\documentclass[sigconf,nonacm]{acmart}

\usepackage{booktabs}       
\usepackage{amsfonts}       
\usepackage{nicefrac}       
\usepackage{microtype}      
\usepackage[table]{xcolor}
\usepackage{comment}
\usepackage{amsmath}
\usepackage{amsfonts}
\usepackage{mathtools}
\usepackage{subcaption}
\usepackage{bm}
\usepackage{bbm}
\usepackage{makecell}
\usepackage{pdflscape}
\usepackage{caption} 
\usepackage{longtable}
\usepackage{graphicx}
\usepackage{dsfont}

\newtheorem{theorem}{Theorem}[section]
\newtheorem{proposition}{Proposition}[section]

\newtheorem{lemma}[theorem]{Lemma}

\usepackage{algorithm}
\usepackage{algpseudocode}

\newcommand{\One}{\mathbbm{1}}

\newcommand{\expectation}[1]{\mathbb{E} \left[ #1 \right]} 
  
\newcommand{\expectationD}[1]{\mathbb{E}_D \left[ #1 \right]}  
\newcommand{\expectationDtrain}[1]{\mathbb{E}_{D_{train}} \left[ #1 \right]} 
\newcommand{\expectationDvalid}[1]{\mathbb{E}_{D_{valid}} \left[ #1 \right]}

\DeclareMathOperator*{\argmin}{arg\,min}

\newif\ifshowcomments
\showcommentstrue 

\ifshowcomments
\newcommand{\changed}[1]{\textcolor{black}{#1}}
\newcommand{\dhe}[1]{\textcolor{blue}{DH: #1}}
\newcommand{\dhc}[1]{\textcolor{blue}{[DH comment: #1]}}
\newcommand{\lpe}[1]{\textcolor{violet}{[LP: #1]}}
\newcommand{\lpc}[1]{\textcolor{violet}{[LP comment: #1]}}
\newcommand{\noe}[1]{\textcolor{teal}{[NO: #1]}}
\newcommand{\noc}[1]{\textcolor{teal}{[NO comment: #1]}}
\newcommand{\nte}[1]{\textcolor{orange}{[NT: #1]}}
\newcommand{\ntc}[1]{\textcolor{orange}{[NT comment: #1]}}
\newcommand{\fle}[1]{\textcolor{red}{[FL: #1]}}
\newcommand{\flc}[1]{\textcolor{red}{[FL comment: #1]}}
\newcommand{\mve}[1]{\textcolor{green}{[MV: #1]}}
\newcommand{\mvc}[1]{\textcolor{green}{[MV comment: #1]}}
\else
  \newcommand{\changed}[1]{}
  \newcommand{\dhe}[1]{}
  \newcommand{\dhc}[1]{}
  \newcommand{\lpe}[1]{}
  \newcommand{\lpc}[1]{}
  \newcommand{\nte}[1]{}
  \newcommand{\ntc}[1]{}
  \newcommand{\noc}[1]{}
  \newcommand{\noe}[1]{}
  \newcommand{\fle}[1]{}
  \newcommand{\flc}[1]{}
  \newcommand{\mve}[1]{}
  \newcommand{\mvc}[1]{}
\fi
\newcommand{\ourmethod}[0]{\textsc{MCGrad}}

\hypersetup{
    colorlinks=true,
    linkcolor=blue,
    citecolor=blue,
    urlcolor=blue
}

\renewcommand\footnotetextcopyrightpermission[1]{}
\begin{document}

\title{\ourmethod{}: Multicalibration at Web Scale}

\author{Niek Tax}
\affiliation{%
  \institution{Meta Platforms, Inc.}
  \city{London}
  \country{United Kingdom}
}
\email{niek@meta.com}

\author{Lorenzo Perini}
\affiliation{%
  \institution{Meta Platforms, Inc.}
  \city{London}
  \country{United Kingdom}
}
\email{lorenzoperini@meta.com}

\author{Fridolin Linder}
\affiliation{%
  \institution{Meta Platforms, Inc.}
  \city{Munich}
  \country{Germany}
}
\email{flinder@meta.com}

\author{Daniel Haimovich}
\affiliation{%
  \institution{Meta Platforms, Inc.}
  \city{London}
  \country{United Kingdom}
}
\email{danielha@meta.com}

\author{Dima Karamshuk}
\affiliation{%
  \institution{Meta Platforms, Inc.}
  \city{London}
  \country{United Kingdom}
}
\email{karamshuk@meta.com}

\author{Nastaran Okati}
\affiliation{%
  \institution{Max Planck Institute for Software Systems}
  \city{Kaiserslautern}
  \country{Germany}
}
\email{nastaran@mpi-sws.org}

\author{Milan Vojnovic}
\affiliation{%
  \institution{Meta Platforms, Inc.}
  \city{London}
  \country{United Kingdom}
}
\affiliation{%
  \institution{The London School of Economics and Political Science}
  \city{London}
  \country{United Kingdom}
}
\email{m.vojnovic@lse.ac.uk}

\author{Pavlos Athanasios Apostolopoulos}
\affiliation{%
  \institution{Meta Platforms, Inc.}
  \city{Menlo Park}
  \state{California}
  \country{USA}
}
\email{pavlosapost@meta.com}

\renewcommand{\shortauthors}{Niek Tax et al.}

\begin{abstract}
We propose \ourmethod{}, a novel and scalable multicalibration algorithm. Multicalibration - calibration in subgroups of the data - is an important property for the performance of machine learning-based systems. Existing multicalibration methods have thus far received limited traction in industry. We argue that this is because existing methods (1) require such subgroups to be manually specified, which ML practitioners often struggle with, (2) are not scalable, or (3) may harm other notions of model performance such as log loss and Area Under the Precision-Recall Curve (PRAUC). \ourmethod{} does not require explicit specification of protected groups, is scalable, and often improves other ML evaluation metrics instead of harming them. \ourmethod{} has been in production at \textbf{Meta}, and is now part of hundreds of production models. We present results from these deployments as well as results on public datasets. \changed{We provide an open source implementation of \ourmethod{} at \url{https://github.com/facebookincubator/MCGrad}}.
\end{abstract}


\maketitle

\section{Introduction}
\label{sec:introduction}



A machine learning model is said to be \emph{calibrated} when its predictions match the true outcome frequencies~\cite{silva2023classifier,kull2017beta,bella2010calibration}. The importance of calibration to ML-based systems has been widely recognized~\cite{Kendall2017uncertainties, Guo2017oncalibration, Kull2019beyond}, including in web applications such as content ranking~\cite{Kweon2022obtaining, Penha2021}, recommender systems~\cite{Gulsoy2025, daSilva2025, Steck2018}, digital advertising~\cite{chen2022calibrating, Chaudhuri2017, Fan2023} and content moderation~\cite{Liu2025, Vidgen2020, Kivlichan2021}. Calibration is necessary for such systems to make optimal decisions without under- or overestimating risk or opportunity.

Multicalibration is a more powerful property that extends the concept of calibration to ensure that predictors are simultaneously calibrated on various (potentially overlapping) groups~\cite{pmlr-v80-hebert-johnson18a,pfisterer2021,noarov2023statistical,jung2021moment}. Often, multicalibration is applied as a post-processing step to fix a base classifier's predictions. While it originated in the study of algorithmic fairness, its significance extends far beyond it. A growing body of work shows that multicalibration can increase model performance and robustness in many circumstances, ranging from out-of-distribution prediction to confidence scoring of LLMs ~\cite{Gopalan2022LowDegreeM, Yang2024, wu2024bridging, kim2022universal, detommaso2024multicalibration, dwork2022beyond}. 

Despite its potential, multicalibration has received limited traction in industry. We argue that this is due to three main reasons. \textbf{First}, existing multicalibration methods require \textit{protected groups} to be \emph{manually defined}~\cite{jin2025}: not only do users need to specify the covariates that require protection (e.g., \texttt{user age} and \texttt{user country}), but also they must define concrete protected groups through binary membership indicators, such as ``\texttt{is an adult in the US?}''.
This poses a significant challenge to real-world application of multicalibration because practitioners (a) may not have established the precise set of protected groups, (b) may be subject to changes to the definition of protected groups over time thus requiring ongoing effort and change management (e.g., due to changes related to legal context, policy frameworks, ethical considerations, or otherwise), or (c) may only seek an overall performance improvement without prioritizing specific groups for protection. The difficulties defining protected groups become even more pronounced when the set of features that we want to protect is large. \textbf{Second}, existing methods lack the ability to scale to large datasets and/or many protected groups, which makes them hard to deploy in production. For instance, \citet{pmlr-v80-hebert-johnson18a}'s algorithm scales at least linearly in time and memory with the number of groups, being potentially inefficient when a large set of groups is specified. To be deployed at web scale, multicalibration methods must be highly optimized so as to not introduce significant computational requirements at training or inference time. 
\textbf{Third}, existing multicalibration methods lack guardrails to be safely deployed, and risk harming model performance (e.g., due to overfitting), which may prevent practitioners from deploying them altogether.

In this paper, we introduce \ourmethod{} (\underline{M}ulti\underline{C}alibration \underline{Grad}ient Boosting) a novel multicalibration algorithm deployed in production that is lightweight and safe to use.
\ourmethod{} requires only the specification of a set of protected features, rather than predefined protected groups. It then identifies miscalibrated regions within this feature space, and calibrates a base classifier's predictions over all groups that can be defined based on these features.
On a high level, \ourmethod{} uses a Gradient Boosted Decision Tree (GBDT) algorithm recursively, such that the result converges to a multicalibrated predictor (Section~\ref{sec:achieving_multicalibration}). By using a highly optimized GBDT implementation, like LightGBM~\cite{lightgbm}, \ourmethod{} inherits scalability and regularization (Section~\ref{sec:scalability}). Finally, it employs an early stopping procedure to avoid overfitting, ensuring that model performance is not harmed (Section~\ref{sec:safe_deployment}).

Empirically, on benchmark datasets, we show that \ourmethod{} improves the base predictor's outputs in terms of multicalibration, calibration, and even predictive performance (Section~\ref{sec:benchmark}). We complement this analysis by describing the impact of \ourmethod{} on real-world data: \ourmethod{} has been used in production at \textbf{Meta} on hundreds of machine learning models, and in total, generates over a million multicalibrated real-time predictions per second (Section~\ref{sec:experiments}). 
\changed{\ourmethod{} successfully addresses all three previous challenges and, to the best of our knowledge, this is the largest-scale adoption of multicalibration in production. Our deployment results show significant impact: on $24$ of $27$ models tested via A/B testing on our Looper platform~\cite{markov2022looper}, \ourmethod{} significantly outperformed Platt scaling, leading to promotion of the \ourmethod{}-calibrated variants to production. Additionally, across $120$+ models on our internal ML platform, \ourmethod{} improved log loss for $88.7\%$ of models, PRAUC for $76.7\%$, and Expected Calibration Error for $86.0\%$.} With these real-world results, we add to the growing body of evidence that shows that multicalibration can significantly contribute to model performance. We also provide valuable practical learnings from applying multicalibration in industry (Section~\ref{sec:learnings}), and link multicalibration to related areas (Section~\ref{sec:related}).

\section{Background on (Multi)Calibration}
\label{sec:background_calibration}

Let $\mathcal{X} = \mathbb{R}^d$ be a $d$-dimensional input space and $\mathcal{Y} = \{0, 1\}$ the binary target space. Let $D = \{(x_i,y_i)\}_{i=1}^n \sim p(X,Y)$ be the dataset with $n$ i.i.d. samples, where the random variables $X$, $Y$ represent, respectively, a $d$-dimensional feature vector and the target. A probabilistic predictor $f$ is a map $f \colon \mathcal{X} \to [0,1]$, which assigns an instance $x \sim p(X)$ to an estimate of its class conditional probability $p(Y=1\mid X=x)$. For any probabilistic predictor $f$, we denote with $F$ its logit (referred to as the predictor), i.e., $f(x) = \mu(F(x))$, where $\mu(F) = 1/(1+e^{-F})$ is the sigmoid function. We assume a base probabilistic predictor $f_0$ is given; this predictor is the one targeted for calibration via a post-processing procedure.

Let $\mathcal{L} \colon \mathbb{R} \times \mathcal{Y} \to \mathbb{R}$ be the log loss (negative log-likelihood), i.e., for any $(x,y)\in \mathcal{X} \times \mathcal{Y}$ and any predictor $F$,
\begin{equation*}
    \mathcal{L}(F(x),y) = -[y\log (\mu(F(x))) + (1-y)\log (1-\mu(F(x)))].
\end{equation*}

We use the symbol $\mathbb{E}$ for the expectation with respect to an arbitrary distribution that is clear from the context, and $\mathbb{E}_D$ for the sample mean over $D$.

\paragraph{\textbf{Calibration}}
A predictor $f$ is perfectly calibrated if and only if $\mathbb{P}(Y=1 \mid f(X)=p)=p$, where $p$ is the true underlying probability. Intuitively, for all input pairs $(x,y)\in D$, if $f$ predicts $0.8$, we expect that $80\%$ of them have $1$ as label. Recently, \citet{tygert2023calibration} introduced the Estimated Cumulative Calibration Error (ECCE), a novel parameter-free metric for measuring calibration that computes the maximum difference between the means of labels and predictions over any score interval: 
\begin{equation}\label{eq:ecce_definition_pscale}
\textsc{ECCE}(f) = \frac{1}{n} \max_{1\leq i\leq j\leq n} \big|\sum_{k=i}^j (y_{(k)}-f(x_{(k)}))\big|,
\end{equation}
where $(x_{(1)}, y_{(i)}), \ldots, (x_{(n)}, y_{(n)})$ are the points $\{(x_i, y_i)\}_{i\in D}$ ordered such that $f(x_{(1)})\leq \cdots \leq f(x_{(n)})$. 

This metric enjoys useful statistical properties. In particular, it can be standardized by dividing it by the following scaling statistic:\footnote{Notice that $\sigma(f)$ is the standard deviation of the sample mean of prediction labels with prediction probabilities according to $f$. For a calibrated predictor $f$, the expected value of $\nicefrac{\text{ECCE}(f)}{\sigma(f)}$ scales as $c$, for large data size $n$, with $c = 2\sqrt{2/\pi}\approx 1.6$ \cite{arrieta2022metrics}.}
\begin{equation*}
\sigma(f) = \sqrt{\frac{1}{n^2} \sum_{k=1}^n f(x_{k})(1-f(x_{k}))}.
\end{equation*}

 This gives rise to two ways of interpreting the metric. The absolute scale $\text{ECCE}(f)$ is designed to measure the magnitude of miscalibration. When the data size increases, $n \rightarrow \infty$, $\text{ECCE}(f)$ converges to a true magnitude of miscalibration ($0$ if $f$ is calibrated and a positive value otherwise). The sigma scale $\nicefrac{\text{ECCE}(f)}{\sigma(f)}$ however converges to a distribution when $f$ is calibrated, and  diverges to infinity when it is miscalibrated. The two scales are used to answer different questions. $\text{ECCE}$ is used to answer the question "is the miscalibration large?". $\nicefrac{\text{ECCE}(f)}{\sigma(f)}$ is used to answer the question "is there statistical evidence of miscalibration?". For example, we might have that $\text{ECCE}(f_1) = 0.05 = 20 \sigma(f_1)$ and $\text{ECCE}(f_2) = 0.05 = 2 \sigma(f_2)$. This would tell us that while the magnitude of miscalibration on the data appears to be similar for $f_1,f_2$, the result is statistically significant for $f_1$ but not for $f_2$, which could simply be due to random sampling. See \cite{arrieta2022metrics, tygert2023calibration} for further discussion.
 
\paragraph{\textbf{Multicalibration}}
Unlike calibration, multicalibration considers the intersection of the \textit{score interval} and the \textit{protected groups}. Let $\mathcal{G}\subset\{g:[0,1]\rightarrow \{0,1\}\}$ be the set of \emph{interval membership functions}, i.e. $g(v) = \One_I(v)$ for $v\in [0,1]$, and any interval $I\subseteq[0,1]$. Similarly, let $\mathcal{H} \subseteq \{h: \mathcal{X} \rightarrow \{0,1\}\}$ be a set of \emph{group membership functions} over $\mathcal{X}$, such that a point $x\in \mathcal{X}$ belongs to a group $h\in \mathcal{H}$ if and only if $h(x) = 1$. The groups can be overlapping, i.e.,  some point $x\in \mathcal{X}$ may belong to multiple groups. 

Multicalibration is defined in various ways across the literature~\cite{pmlr-v80-hebert-johnson18a,shabat2020sample,blasiok_et_al:LIPIcs.ITCS.2024.17,Gopalan2022LowDegreeM,happymap23,haghtalab2023}. Here, we employ an operational definition of multicalibration:
\begin{definition}[Multicalibration]~\label{def:multicalibration}
    A probabilistic predictor $f$ is $\alpha$-multicalibrated ($\alpha$-MC) with respect to $\mathcal{H}$ if, for the MC-deviation given as $\Delta_{h,g}(f) \coloneqq \big| \expectation{h(X) g(f(X)) (Y-f(X))} \big|$ and a given scale parameter $\tau_h(f)$:
\begin{equation*}
    \Delta_{h,g}(f) \le \alpha\tau_h(f) \quad \forall h \in \mathcal{H}, g \in \mathcal{G}.
\end{equation*}
We say that $f$ is \emph{multicalibrated} if $\alpha = 0$. 
\end{definition}

A common choice is to use a uniform bound on the MC-deviation, which is accommodated by choosing the scale parameter $\tau_h(f)$ to be a constant equal to $1$. In Appendix~\ref{sec:theory}, we show that Definition~\ref{def:multicalibration}, with $\tau_h(f)^2=\expectation{h(X)f(X)(1-f(X))}$, naturally matches with Multicalibration Error (MCE) by \citet{guy2025measuringmulticalibration}, which quantifies multicalibration error by computing the maximum ECCE over the protected groups with a suitable normalization: 
\begin{equation}\label{equ:mcedef}
\text{MCE}(f)=\max_{h\in \mathcal{H}} 
\frac{\text{ECCE}_h(f)}{\sigma_h(f)}
\end{equation}
where $\text{ECCE}_h(f), \sigma_h(f)$ are defined as $\text{ECCE}(f), \sigma(f)$ above but restricted to the group $h$. Roughly speaking, $\text{MCE}$ measures the strongest statistical evidence of miscalibration across protected groups. The unit of the MCE is the same as for $\text{ECCE}/\sigma$. It can be re-scaled to an absolute measure by multiplying with $\sigma(f)$. 

Using a multicalibration metric that is equivalent to the definition of $\alpha$-MC can help ensure that empirical work is aligned with theory. In Section~\ref{sec:benchmark} we use $\text{MCE}$ to measure multicalibration, thereby directly estimating the minimal $\alpha$ for which a model is $\alpha$-MC. To our knowledge, this is the first empirical multicalibration study that achieves this. See Appendix~\ref{sec:theory} for further discussion.


\section{\ourmethod{}: A Practical Algorithm for Multicalibration}
\label{sec:algorithm}

This paper tackles the following problem: 
\begin{description}
    \item[Given] a labeled dataset $D$, an initial probabilistic predictor $f_0$;
    \item[Return] a probabilistic predictor $f$ that is $\alpha$-MC w.r.t. $\mathcal{H}$.
\end{description}

Designing an algorithm for multicalibration without specifying a set of groups is challenging for four reasons. First, the absence of groups requires the multicalibration model to automatically determine the regions of the feature space where the base model is highly miscalibrated. Second, adjusting the model's output to improve its calibration in a specific region might harm the predictions over other regions, thus increasing the multicalibration error. Third, a post-processing algorithm has to be fast and lightweight to prevent large memory consumption or increased latency of model inference. Fourth, while existing algorithms claim to achieve multicalibration on the training set, they are subject to overfitting issues.

\begin{algorithm}[tb]
\begin{algorithmic}
\Require a probabilistic predictor $f_0$, a dataset $D$; 
\State $D_{train}$, $D_{valid}$ = train\_validation\_split($D$)
\State $F_0 = \mu^{-1}(f_0)$ \# inverse sigmoid transformation
\State $\varepsilon_{-1} = +\infty $, $\varepsilon_0 = \expectationDvalid{\mathcal{L}(F_0(X), Y)}$  \# initialize errors
\State $t=1$ \# initialize number of rounds for early stopping
\While{$\varepsilon_{t-1} - \varepsilon_t > 0$}
\State $h_{t} (x, f_{t-1}(x)) =$ fit GBDT on $\{((x_i, f_{t-1}(x_i)), y_i)\}_{i\in D_{train}}$
\State $\theta_t = \argmin_{\theta} \expectationDtrain{\mathcal{L}(\theta \cdot (F_{t-1}(X)+h_{t}(X, f_{t-1}(X))), Y)}$ 
\State $F_{t}(x) = \theta_t \cdot (F_{t-1}(x)+h_{t}(x, f_{t-1}(x)))$
\State $f_{t}(x) = \mu(F_{t}(x))$
\State $\varepsilon_{t+1} = \expectationDvalid{\mathcal{L}(F_{t}(X), Y)}$
\State $t \leftarrow t+1$
\EndWhile
\State $\#$ Found $t-1$ to be the best number of rounds
\For{$s$ in $[1, \dots, t-1]$}
\State $h_{s} (x, f_{s-1}) =$ fit GBDT on $\{((x_i, f_{s-1}(x_i)), y_i)\}_{i\in D}$
\State $\theta_s = \argmin_{\theta} \expectationD{\mathcal{L}(\theta \cdot (F_{s-1}(X)+h_{s}(X, f_{s-1}(X))), Y)}$ 
\State $F_{s}(x) = \theta_s \cdot (F_{s-1}(x)+h_{s}(x, f_{s-1}(x)))$
\State $f_s(x) = \mu(F_s(x))$
\EndFor
\State $f(x) = \mu(F_{t-1}(x))$
\Ensure $f$;
\end{algorithmic}
\caption{\ourmethod{}}
\label{alg:gbmct_new}
\end{algorithm}

Our proposed algorithm \ourmethod{}, presented in Algorithm~\ref{alg:gbmct_new}, is a lightweight algorithm that multicalibrates a base model without requiring the specification of the groups on which the model has to be calibrated. \ourmethod{} relies on a key observation: by including the base model's predictions as a feature, the loss function of the multicalibration algorithm captures the feature values that correspond to miscalibrated initial predictions. In addition, decreasing the loss function implies correcting the base model's predictions for some regions of the feature space. Because corrections in some regions may negatively impact others, \ourmethod{} runs multiple rounds, using the previous round's processed predictions as input. Since GBDTs are regularized using shrinkage (or a step size $<1$), a simple rescaling step reduces the number of trees and rounds required, while having a negligible impact on overfitting. Finally, \ourmethod{} employs early stopping to avoid overfitting. 

\subsection{\texorpdfstring{Achieving Multicalibration with $T$ Rounds}{Achieving Multicalibration with T Rounds}}\label{sec:achieving_multicalibration}

Our intuition builds on the following two key insights. First, GBDT returns a solution $f$ that approximately sets the gradient of the loss function $\mathcal{L}$ to zero:
\begin{equation*}
    \expectationD{h(X)(Y-f(X))} = 0 \qquad  \forall h \in \mathcal{H},
\end{equation*}
where $\mathcal{H}$ is the set of all regression trees over the feature space.

Second, if we augment the feature space $\mathcal{X}$ with a single additional feature $f_0(x)$, then the GBDT approximately achieves 
\begin{equation}\label{eq:gbdtachievemc}
    \expectationD{h(X,f_0(X))(Y-f_1(X))} = 0 \qquad  \forall h \in \mathcal{H}'
\end{equation}
where $\mathcal{H}' \subset \{h: \mathcal{X}\times [0,1]\rightarrow \mathbb{R}\}$ is the space of trees with an extra real-valued input, and $f_1(x) = \mu(F_0(x) + h_1(x,f_0(x)))$. This creates a bridge between (a) group and interval membership functions and (b) regression trees over the augmented feature space. Strictly speaking, for any $h\in \mathcal{H}, g\in \mathcal{G}$ there exists $h'\in\mathcal{H}'$ such that $h(x)g(f(x)) = h'(x, f(x))$ for $x\in\mathcal{X}$.

Perhaps surprisingly, Eq.~(\ref{eq:gbdtachievemc}) is equivalent to $0$-multicalibration according to Definition~\ref{def:multicalibration} under the condition $f_1\equiv f_0$. However, this condition may not always hold: \emph{correcting $f_0$ for some regions identified by $h(x,f_0(x))$ returns a predictor $f_1$ that might be miscalibrated on regions identified by its own predictions $h(x,f_1(x))$}. 
To resolve this, we create a loop over multiple rounds: in each round $t$, we train a GBDT model $f_t$ using the features $x$ and the previous round's predictions $f_{t-1}$, such that, after $T$ rounds, we have
\begin{equation*}
    \expectationD{h(X, f_{T-1}(X))(Y-f_T(X))} = 0 \qquad  \forall h \in \mathcal{H}^*
\end{equation*}
where $f_T(x) = \mu(F_{T-1}(x) + h_T(x,f_{T-1}(x)))$.

Intuitively, if for some probabilistic predictor $f^*$, $f_{T}\rightarrow f^*$ as $T\rightarrow \infty$, then for all $h \in \mathcal{H}^*$, $\expectationD{h(X, f_{T-1}(X))(Y-f_T(X))}\rightarrow \expectationD{h(X, f^*(X))(Y-f^*(X))} = 0$ as $T\rightarrow \infty$, which is our desired condition. This suggests that, for sufficiently large $T$, the probabilistic predictor $f_T$ should be approximately multicalibrated. See Appendix~\ref{app:conv} for further intuitions and theoretical results.

\subsection{Fast Training and Prediction}\label{sec:scalability}
To be deployed at web scale, there are strict requirements for an algorithm to be efficient both at training and inference time. Several design choices ensure that \ourmethod{} meets these requirements.

\paragraph{\textbf{Efficient gradient boosting}} First, the algorithm has been designed to rely on a relatively small number of calls to a GBDT, delegating the most compute intensive steps to one of many highly optimized GBDT implementations (e.g. ~\cite{lightgbm, xgboost}). This differentiates it from existing implementations such as~\cite{pmlr-v80-hebert-johnson18a}. In our implementation, we use LightGBM.

\paragraph{\textbf{Rescaling the logits.}}
GBDTs are regularized in multiple ways, including using a step size smaller than $1$ for scaling each additional tree. While this helps to avoid overfitting, it results in a predictor which can be improved by rescaling it by a factor slightly greater than $1$. As a result, the next round will attempt to apply this rescaling. Since linear rescaling is not easily expressed by decision trees, this may require many trees, ending up requiring unnecessarily many trees to achieve multicalibration.

We introduce a simple rescaling after every round:
\begin{equation*}
    \theta_t = \argmin_{\theta} \expectationD{\mathcal{L}(\theta \cdot (F_{t-1}(X) + h_t(X,f_{t-1}(X))), Y)}
\end{equation*}
where $h_t$ is obtained with LightGBM. We call the round $t$'s learned predictor (on the logits) as $F_t = \theta_t (F_{t-1}+h_t)$.

Note that this constant is typically very close to $1$ and has a very limited detrimental effect on regularization. Rescaling is not the same as using a step size of $1$: it affects the whole sum of trees (i.e., the whole predictor), while the step size only targets each tree sequentially (i.e., each weak learner).

\subsection{Preventing Overfitting}\label{sec:safe_deployment}
Modern GBDT algorithms support various methods for regularization, such as limiting the growth of the trees, the leaf splits, and the number of trees. \ourmethod{}'s recursive structure can give rise to additional overfitting beyond standard GBDTs, which we address in the following ways.

\paragraph{\textbf{Early stopping}} While multiple rounds are required for convergence, they also increase the capacity of the model. With $T$ rounds and $M$ trees in each round, the \ourmethod{} model is at least as expressive as a tree ensemble with $T
\cdot M$ trees. Since the model capacity of tree ensembles is $\Omega(\text{number of trees})$ (see~\cite{shalevshwartz2013understanding}), overfitting is more likely for a large number of rounds. 

We solve this problem with a standard early stopping procedure on the number of rounds $T$. Specifically, we split i.i.d. the dataset $D$ into training ($D_{train}$) and validation ($D_{valid}$) sets, and determine the number of rounds $T$ by taking the last round before the expected loss over the validation set increases. That is,
\begin{equation*}
    T = \min \{t \colon \expectationDvalid{\mathcal{L}(F_{t+1}(X), Y) - \mathcal{L}(F_{t}(X), Y)} > 0\}-1,
\end{equation*}
where the models $f_t$ are obtained using $D_{train}$ instead of $D$. Using early stopping in a real-world deployed system has a relevant consequence: \emph{\ourmethod{} does not harm the base model's prediction}, and, instead, would select $T=0$ as optimal number of rounds if the first step decreased the initial performance:
\begin{equation*}
    \expectationDvalid{\mathcal{L}(F_{1}(X), Y) - \mathcal{L}(F_0(X), Y)} > 0  \implies  T = 0
\end{equation*}
which means that $f_T = f_0$. 

\paragraph{\textbf{Regularizing through the min sum Hessian in leaf}} 

Augmenting the data with the previous round's model necessarily gives rise to regions of the augmented feature space that are particularly prone to overfitting. Consider a leaf of a partially constructed tree in GBDT. The leaf splitting algorithm can choose to split the leaf on values of the previous predictor $f_t(x)$. In this leaf, the left tail $f_t(x)<a$ for some small $a$ will contain only negative labels even if the true distribution may assign a positive probability to a positive label (and analogously for the right tail). In that case the new model can improve the likelihood simply by assigning a very low probability to this tail, as low as zero. Common regularization strategies, like setting the minimum number of samples per leaf, are insufficient to address this scenario. We solve this problem by using a more targeted form of regularization that limits the minimum total Hessian in a leaf, which is offered as one of multiple regularization techniques in LightGBM. The Hessian in a leaf $S$ equals $\sum_{(x,y) \in S} f_t(x) (1-f_t(x))$ and, as such, is a refined version of the simple sample size rule $|S|$. As the predicted probabilities become close to $0$ or $1$, the total Hessian becomes smaller and split in multiple leaves can no longer be considered. As shown in our ablation studies (Section~\ref{sec:benchmark}) this regularization reduces overfitting and improves performance.
\section{Benchmark Experimental Analysis}
\label{sec:benchmark}
We evaluate~\ourmethod{} on both public and production data. In this section, we present experiments on benchmark datasets and compare~\ourmethod{} against state-of-the-art baselines. In Section~\ref{sec:experiments} we present results of~\ourmethod{} in production. Here, we focus on \changed{five} research questions:\footnote{Code: \url{https://github.com/facebookresearch/mcgrad_multicalibration_at_web_scale}.}
\begin{itemize}
    \item[Q1.] How does \ourmethod{} compare to \emph{existing baselines} on unspecified groups?
    \item[Q2.] Existing methods require manual specification of protected groups. Does \ourmethod{} still protect those specified groups, even though they are not specified in~\ourmethod{}?
    \item[Q3.] Does \ourmethod{} benefit from running multiple rounds?
    \item[Q4.] What is the effect of \emph{rescaling the logits} and the \emph{regularization through min sum Hessian in leaf} on~\ourmethod{}?
    \changed{\item[Q5.] How does \ourmethod{}'s computational time compare to existing multicalibration methods?}
\end{itemize}

\begin{table*}[h]
    \centering
    \small
    \changed{
    \begin{tabular}{l|c|c|c|l}
    \toprule
    Dataset & \# Samples $n$ & \# Features $d$ & Class Distr. & Protected Attributes \\
    \midrule
    MEPS~\cite{sharma2021fair} & 11079 & 139 & 0.169 & Race, Gender, Age, Income, Insurance Status \\
    Credit~\cite{default_of_credit_card_clients_350} & 30000 & 118 & 0.221 & Gender, Age, Education, Marital Status \\
    Marketing~\cite{bank_marketing_222} & 45211 & 41 & 0.116 & Age, Marital Status, Education, Occupation \\
    HMDA~\cite{cooper2023variance} & 114185 & 89 & 0.752 & Race, Gender, Ethnicity, Age \\
    ACSIncomeCA~\cite{ding2021retiring, flood2021integrated} & 195665 & 10 & 0.410 & Race, Gender, Age, Education, Income, Employment \\
    ACSMobility~\cite{ding2021retiring, flood2021integrated} & 616207 & 23 & 0.735 & Race, Gender, Age, Education, Income, Employment \\
    ACSPublic~\cite{ding2021retiring, flood2021integrated} & 1123374 & 24 & 0.293 & Race, Gender, Age, Education, Income, Employment \\
    ACSTravel~\cite{ding2021retiring, flood2021integrated} & 1458542 & 19 & 0.437 & Race, Gender, Age, Education, Income, Employment \\
    ACSIncome~\cite{ding2021retiring, flood2021integrated} & 1655429 & 10 & 0.370 & Race, Gender, Age, Education, Income, Employment \\
    ACSEmploy~\cite{ding2021retiring, flood2021integrated} & 3207990 & 20 & 0.456 & Race, Gender, Age, Education, Income, Employment \\
    ACSHealth~\cite{ding2021retiring, flood2021integrated} & 3207990 & 31 & 0.150 & Race, Gender, Age, Education, Income, Employment \\
    \bottomrule
    \end{tabular}
    \caption{Number of samples, features, class distribution ($p(Y=1)$), and protected attributes for each benchmark dataset.}
    \label{tab:info_datasets}}
    \vspace{-0.5cm}
\end{table*}

\subsection{Experimental Setup}
\changed{\paragraph{\textbf{Data.}} 
Our experimental analysis uses 11 datasets (see Table~\ref{tab:info_datasets}). Six binarized prediction tasks are derived from the American Community Survey (ACS)~\cite{ding2021retiring, flood2021integrated} using the folktable package, plus a California-specific income task (ACSIncomeCA). The remaining five datasets are: UCI Bank Marketing~\cite{bank_marketing_222} (term deposit prediction), UCI Default of Credit Card Clients~\cite{default_of_credit_card_clients_350} (debt default), Home Mortgage Disclosure Act (HMDA)~\cite{cooper2023variance} (mortgage acceptance), and Medical Expenditure Panel Survey (MEPS)~\cite{sharma2021fair} (medical visits). For all datasets we use the same protected groups as in~\cite{hansen2024multicalibration}. Protected attributes are selected following standard fairness benchmarks.}
 
\paragraph{\textbf{Baselines.}} We compare \ourmethod{} against two multicalibration baselines: Discretization-Free MultiCalibration (DFMC)~\cite{jin2025} and \textsc{HKRR}~\citet{pmlr-v80-hebert-johnson18a}. DFMC shares some similarities with \ourmethod{}: It fits a single GBDT with a fixed maximum depth of 2 for each weak learner. Protected groups are explicitly specified by the user, rather than learned like in~\ourmethod{}. Protected groups are provided to the algorithm as binary features that encode group membership. \textsc{HKRR} is the boosting procedure proposed by~\citet{pmlr-v80-hebert-johnson18a}, for which we use ~\citet{hansen2024multicalibration}'s implementation\footnote{\url{https://github.com/dutchhansen/empirical-multicalibration}}. We also include Logistic Regression as \textsc{BasePred}, the base predictor $f_0$ shared by all methods, and \textsc{Isotonic}~\cite{zadrozny2002transforming} (Isotonic Regression), a widely adopted calibration algorithm.

\noindent\paragraph{\textbf{Metrics.}} We evaluate \ourmethod{} relative to baselines on three dimensions: the effect on 1) \emph{predictive performance}, 2) \emph{calibration}, and 3) \emph{multicalibration}.

1) \emph{Predictive performance}. We employ the Area Under Precision-Recall Curve (PRAUC)~\cite{davis2006relationship} and the log loss.

2) \emph{Calibration}. Various calibration metrics have been proposed in the literature. The \emph{Expected Calibration Error (ECE)}~\cite{naeini2015obtaining} is the most commonly used calibration measure in machine learning, which bins the model scores and computes the deviation from perfect calibration, i.e., the absolute difference between the model’s predicted accuracy and its empirical accuracy~\cite{naeini2015obtaining}. Binning-based calibration metrics, including ECE, vary significantly based on the choice of bins. This is a well-documented shortcoming~\cite{arrieta2022metrics,roelofs2022mitigating}. Alternative metrics that aim to circumvent the drawbacks of binning include \emph{smooth ECE (smECE)} which relies on kernel smoothing~\cite{blasiok2023smooth}, but merely replaces sensitivity to the arbitrary choice of binning with sensitivity to an arbitrary choice of kernel~\cite{arrieta2022metrics}. The \emph{Estimated Cumulative Calibration Error (ECCE)}~\cite{arrieta2022metrics} is a calibration metric that is rooted in cumulative statistics that does not require making an arbitrary choice in bin or kernel, which we covered in more detail in Section~\ref{sec:background_calibration}. \emph{Brier score} is a proper scoring rule that is commonly used to evaluate probabilistic forecasts that is essentially just the mean squared error of predicted probability. Brier score decomposes into calibration and the model's ability to separate the positive and negative class~\cite{blattenberger1985separating}. In this paper, we quantify calibration error using ECCE, with exceptions in Section~\ref{sec:experiments}, where for pragmatic reasons we resort to ECE and Brier score whenever ECCE was not measured in the production system.

3) \emph{Multicalibration}. \citet{blasiok_et_al:LIPIcs.ITCS.2024.17} proposed the \emph{maximum group-wise smECE} metric, which was used for empirical evaluations in \citet{hansen2024multicalibration}. \citet{guy2025measuringmulticalibration} proposed an extension of \emph{ECCE} to quantify multicalibration, called \emph{Multicalibration Error (MCE)}, which we defined in Eq.~(\ref{equ:mcedef}). \emph{Maximum group-wise smECE} and \emph{MCE} have in common that they both are defined as a maximum of some quantity over groups, but differ in the calibration quantity calculated in each group, where they respectively use smECE and \emph{ECCE}. In this paper, we use MCE to quantify multicalibration. \changed{
In the online supplement\footnote{\url{https://github.com/facebookresearch/mcgrad_multicalibration_at_web_scale}}}, we additionally report group-wise smECE, thereby providing results that are in-line with~\cite{hansen2024multicalibration}. We can measure any multicalibration metric with respect to either of two sets of groups. \textit{Prespecified Groups}: refers to a small (less than $15$ across all datasets) set of protected groups that are formed using various features of the datasets. For this, we use the groups specifications that were used in ~\citet{hansen2024multicalibration}. \textit{Unspecified Groups}: covers the case where the user does not specify which groups to protect against. In that case, the scores are expected to be multicalibrated with respect to all possible groups, and to calculate this, an extensive set of combinatorially generated groups are used.


\paragraph{\textbf{Hyperparameters.}}

\changed{\ourmethod{} uses \textsc{LightGBM} as GBDT implementation. To ensure seamless adoption by ML engineers, we set default hyperparameters so that \ourmethod{} works out-of-the-box, without requiring application-specific tuning in the majority of use-cases. We determined the default hyperparameters by performing a grid search over $35$ company-internal datasets and selecting the configuration that never degraded the base model's log loss or PRAUC, and minimized the average MCE across datasets. 
The resulting hyperparameters are: \textsc{learning\_rate}$ = 0.02873$, \textsc{max\_depth}$= 5$, \textsc{min\_child\_samples} $= 160$, \textsc{n\_estimators}$= 94$, \textsc{num\_leaves}$= 5$, \textsc{lambda\_l2}$= 0.00913$, \textsc{min\_gain\_to\_split}$= 0.15$. All other parameters are set as for \textsc{LightGBM} defaults.}

For \textsc{HKRR} we follow~\citet{hansen2024multicalibration} and pick the best hyperparameters from a set of four specifications using a held-out validation set. For DFMC we use the default hyperparameters of \textsc{LightGBM} since the paper does not suggest any specific hyperparameters.

\subsection{Experimental Results}

\begin{figure*}[ht]
    \centering
    \includegraphics[width=0.85\textwidth]{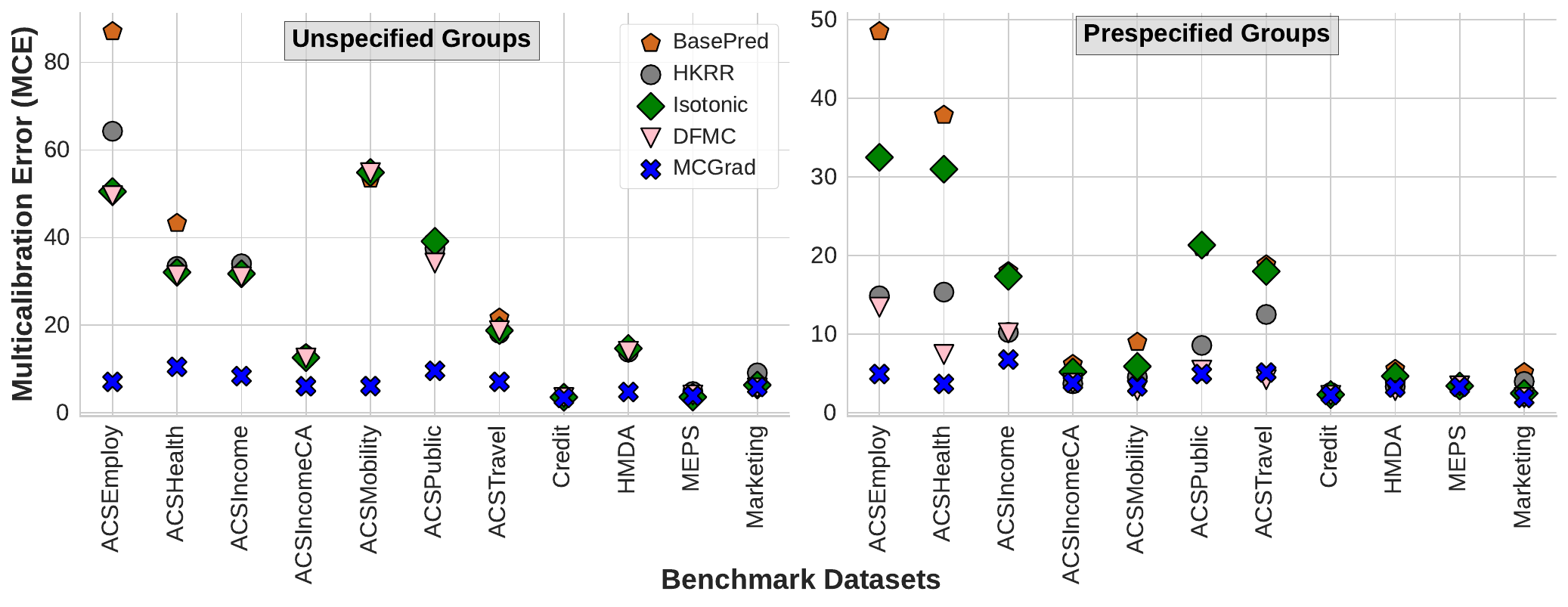}
    \vspace{-0.2cm}
    \caption{Multicalibration Error computed using \emph{unspecified} groups (left) or the manually \emph{prespecified} groups (right) for all compared methods on each benchmark dataset. Overall, \ourmethod{} achieves a better (lower) error for $10$ out of $11$ datasets when tested on unspecified groups, and for $5$ out of $11$ datasets when tested on prespecified groups.}
    \label{fig:mce_sigmascale}
\end{figure*}

\begin{figure*}[ht]
    \centering
    \vspace{-0.1cm}
    \includegraphics[width=0.91\textwidth]{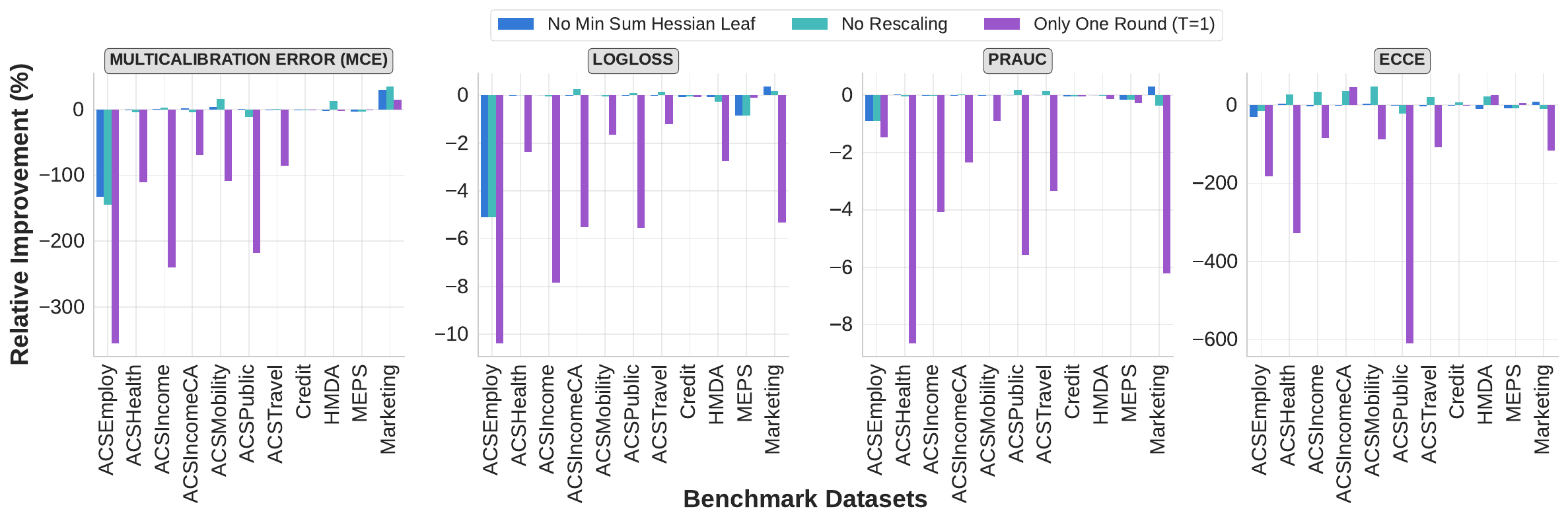}
    \vspace{-0.2cm}
    \caption{Improvement (in \%) of \ourmethod{}'s variants relative to the original version. While setting $T = 1$ yields a significant drop in performance, the effect of the rescaling factor and min sum Hessian in leaf are mild due to the datasets limited size.}
    \label{fig:ablation}
\end{figure*}

\setlength{\tabcolsep}{2.5pt} 
\begin{table}[tb]
\begin{tabular}{l|cc|cc|cc|cc}
\toprule
Baseline & \multicolumn{2}{c|}{MCE} & \multicolumn{2}{c|}{log loss} & \multicolumn{2}{c|}{PRAUC} & \multicolumn{2}{c}{ECCE} \\
 & \textsc{Avg} & \textsc{Rank} & \textsc{Avg} & \textsc{Rank} & \textsc{Avg} & \textsc{Rank} & \textsc{Avg} & \textsc{Rank} \\
\midrule
MCGrad & \textbf{6.60} & \textbf{1.18} & \textbf{0.375} & \textbf{1.00} & \textbf{0.714} & \textbf{1.00} & 1.97 & 2.27 \\
Isotonic & 24.00 & 2.91 & 0.410 & 3.27 & 0.662 & 3.91 & \textbf{1.53} & \textbf{1.54} \\
DFMC & 24.95 & 3.00 & 0.408 & 2.27 & 0.671 & 2.18 & 2.20 & 2.91 \\
HKRR & 25.74 & 3.73 & 0.414 & 4.00 & 0.637 & 5.00 & 3.29 & 3.54 \\
BasePred & 29.07 & 4.18 & 0.414 & 4.45 & 0.668 & 2.91 & 4.66 & 4.73 \\
\bottomrule
\end{tabular}
\caption{Average value (\textsc{Avg}) and rank (\textsc{Rank}) over the datasets, for various metrics. \ourmethod{} has the best average log loss and PRAUC, and second-best average ECCE.}
\label{tab:metrics_comparison_baselines}
\vspace{-0.5cm}
\end{table}

\begin{figure}[ht]
    \centering
    \includegraphics[width=0.485\textwidth]{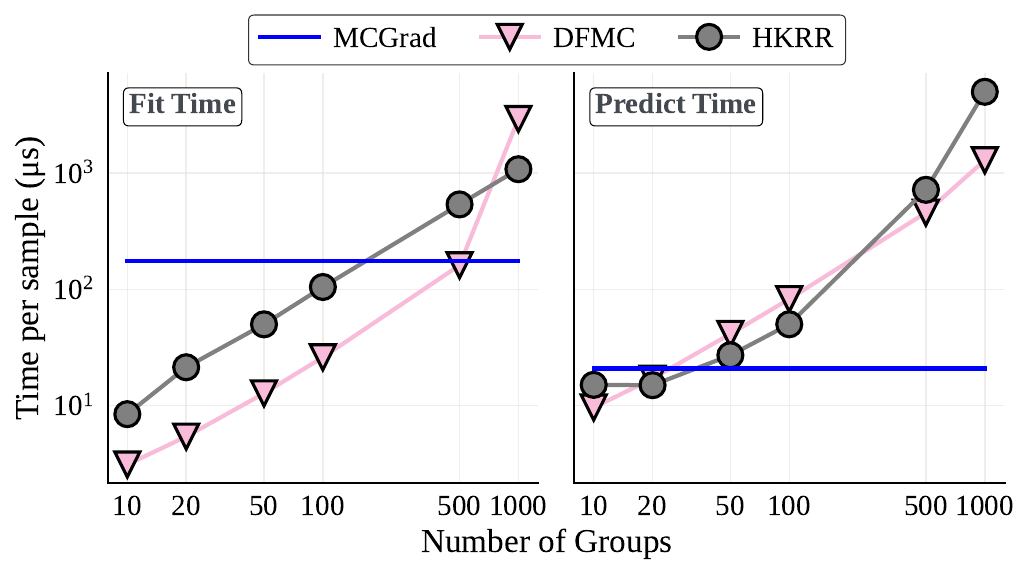}
    \vspace{-0.5cm}
    \caption{\changed{Computational time of multicalibration algorithms with varying numbers of groups. (Left) Fit time and (Right) predict time in seconds as a function of the number of groups on a log-log scale. Runtime for DFMC and HKRR increases with the number of groups, while MCGrad is constant.}}
    \label{fig:computational_cost}
\end{figure}


\paragraph{\bf Q1. \ourmethod{} vs existing baselines on \emph{unspecified} groups.} Figure~\ref{fig:mce_sigmascale} (left-hand side) shows the Multicalibration Error (MCE) for all compared methods using all available features. Overall, \ourmethod{} outperforms all baselines on $10$ out of the $11$ datasets by reducing the base predictor's MCE between $4\%$ (Credit) and $93\%$ (ACSEmploy), and obtaining an average reduction of $56.1\%$. In contrast, \textsc{DFMC}, \textsc{Isotonic}, and \textsc{HKRR} only achieve an average reduction of, respectively, $11.9\%$, $9.6\%$, and $1.4\%$. In addition, \ourmethod{} never harms the base predictor's performance: even on MEPS, where there is low evidence of \textsc{BasePred}'s miscalibration (3.9 MCE), it slightly improves the MCE by $2.8\%$. On the other hand, \textsc{Isotonic}, \textsc{DFMC}, and \textsc{HKRR} increase (worsen) the MCE of \textsc{BasePred} on $3$, $3$ and $4$ datasets respectively.

Table~\ref{tab:metrics_comparison_baselines} reports the average performance over the datasets. \ourmethod{} achieves by far the best average log loss improvement ($10.4\%$) and PRAUC improvement ($8.1\%$) over the base predictor. Besides these average improvements over datasets, \ourmethod{} never harms performance, which helps build the necessary trust among ML engineers to deploy it in production environments. Other methods don't have this guarantee. For instance, while \textsc{DFMC} improves model performance metrics by a lower $1.8\%$ and $0.76\%$, it harms PRAUC and log loss on two datasets. \textsc{Isotonic} and \textsc{HKRR} even decrease the PRAUC on average by, respectively, $1\%$ and $5\%$ with a negligible change in log loss. For each metric, we rank all compared methods from $1$ (best) to $5$ (worst) and report the average ranking position in Table~\ref{tab:metrics_comparison_baselines}. A lower average rank indicates how often the method is preferred, as opposed to the average performance, which reflects the magnitude. \ourmethod{} consistently ranks as $\#1$ baseline on log loss and PRAUC, while only ranking third for MCE on the MEPS dataset. Finally, \ourmethod{} ranks second best method in terms of global calibration (ECCE), losing against Isotonic Regression, which is expected because it is trained to directly optimize the global calibration of the base predictor.

\paragraph{\bf Q2. \ourmethod{} vs existing baselines on \emph{prespecified} groups.} To verify that good performance on the MCE metric for all groups implies also protection for prespecified groups, we report results just for these groups in Figure~\ref{fig:mce_sigmascale} (right-hand side). Overall, \ourmethod{} outperforms all baselines on $5$ out of $11$ datasets achieving a $54\%$ average MCE improvement, and is followed by \textsc{DFMC} that is the best baseline on three datasets (\textsc{ACSMobility}, HDMA, \textsc{ACSTravel}) and achieves an average MCE improvement of $51\%$. \textsc{HKRR} obtains an average improvement in MCE of $36\%$, while \textsc{Isotonic} only achieves an improvement of $16\%$. Note that \textsc{HKRR} and \textsc{DFMC} have direct access to these groups, while \ourmethod{} only has access to the features used to define the groups.

\paragraph{\bf Q3. The impact of multiple rounds on \ourmethod{}} Because \ourmethod{} recurrently improves its own predictions by using $T$ rounds, one relevant question is whether $T>1$ is really necessary in practice. For this goal, we compare \ourmethod{} with its variant \textsc{Only One Round} that forces the number of rounds $T = 1$.\footnote{Early stopping can possibly return $T=0$.} Figure~\ref{fig:ablation} (purple line) shows the relative improvement of \textsc{Only One Round} for the four considered metrics and all benchmark datasets. In the majority of cases, it is evident that \textsc{Only One Round} deteriorates \ourmethod{}'s performance: on the MCE and ECCE axes, it more than doubles its value (relative improvement $< -100$) for, respectively, $6$ and $5$ datasets, while, on the log loss and PRAUC axes, it worsens the performance by approximately $2\%$ or more on, respectively, $8$ and $6$ out of $11$ datasets. Interestingly, the largest deterioration of performance corresponds to \ourmethod{} running for the highest number of rounds: the optimal $T$ for \textsc{ACSEmploy}, \textsc{ACSIncome}, \textsc{ACSPublic}, and \textsc{ACSMobility} is, respectively, $20$, $35$, $27$ and $22$.

\paragraph{\bf Q4. Impact of \emph{rescaling the logits} and the \emph{regularization through min sum Hessian in leaf}}

\emph{Rescaling the logits} of all predictions by learning a constant speeds up \ourmethod{} training by reducing rounds, especially on very large datasets. As a downside, it might affect~\ourmethod{}'s performance. Figure~\ref{fig:ablation} (light blue) shows the relative deterioration after removing rescaling. Effects are mild and balanced: removal worsens MCE, log loss, and PRAUC on $6$, $6$, and $7$ datasets respectively, with average drops of $9\%$, $0.5\%$, $0.03\%$. 

LightGBM’s default \emph{min sum Hessian leaf} (MSHL) is $0.001$ with a min leaf sample size of $20$. We claim that higher MHL values prevent probabilities from pushing to extremes, reducing overfitting (Section~\ref{sec:safe_deployment}). We compare \ourmethod{} with MHL$=0.001$ versus our default of $20$. Figure~\ref{fig:ablation} (dark blue bar) shows mild deterioration using $0.001$, likely due to limited dataset sizes. On average, lower MHL slightly harms performance, but this becomes evident on \textsc{ACSEmploy} (the largest dataset), where the MHL variant runs only $2$ rounds (vs. $10$) because our regularization mechanism stops the training earlier to avoid overfitting.

\paragraph{\bf Q5. Computational time of \ourmethod{} vs multicalibration baselines as the number of groups scales.}

\changed{Scalability can be a relevant bottleneck for multicalibration deployment in production environments. To evaluate this, we measure the training (fit) and inference (predict) times of \ourmethod{}, \textsc{HKRR}, and \textsc{DFMC} as the number of protected groups increases from $10$ to $1000$, using $\approx 1M$ training and $\approx 300k$ test samples from the ACSEmploy dataset.}

\changed{Figure~\ref{fig:computational_cost} shows the methods' computational time per sample (in microseconds) as a function of the number of groups. Notably, the gap between \ourmethod{} and each baseline becomes larger as group count increases: at $1000$ groups, \ourmethod{} maintains nearly constant inference time ($\sim 20\mu s$) while HKRR and DFMC require, respectively, $5019$ and $1292$ $\mu s$, meaning \ourmethod{} achieves less than $2\%$ of their latency overhead. This gap becomes even larger in production settings where models with $100$s of features can generate tens of thousands of potential protected groups.}

\changed{More importantly, these seemingly small latency differences have strong business impact. Production models typically serve online inference where even millisecond-scale increases in response time negatively affect the system performance and the user experience. Existing research claims that an \emph{increase} of latency directly correlates with \emph{reduced} user engagement and online business metrics~\cite{kohavi2014seven}. \ourmethod{}'s sub-$20 \mu s$ inference overhead allows deployment in production, while baseline methods requiring few milliseconds per prediction can be impractical for production use.
}

\section{Results in Production Deployment}
\label{sec:experiments}
We describe the results obtained by using \ourmethod{} in a real-world production system. 
 \ourmethod{} has been deployed and used in production at Meta. Specifically, we have integrated \ourmethod{} into two ML training platforms at Meta: the publicly released \textsc{Looper}~\cite{markov2022looper}, and an internal platform that we refer to as \textsc{MLplatform\#2}. The evaluation consists of a total of $147$ binary classification models that serve live production traffic: $27$ models on Looper and $120$ models on the other platform. On both ML training platforms, we run \ourmethod{} with default hyperparameters.

\paragraph{\textbf{\textsc{Looper} Results.}}
\textsc{Looper}~\cite{markov2022looper} is a system that automates model training, online A/B test comparisons between the new model variant and the existing production model, and launch decisions on whether or not to promote the new model variant. These launch decisions are made based on whether the new model variant outperforms the existing production model on the selected evaluation metric with statistical significance.

For 27 binary classification models that serve active production traffic on \textsc{Looper}, we created a variant of the model that applies~\ourmethod{} as calibration post-processing. We started online A/B tests to compare these 27 models against the active production models, all of which are GBDTs. Those production models all had previously already applied Platt scaling~\cite{platt1999probabilistic}.

We found that on $24$ out of $27$ models, the variant with \ourmethod{} statistically significantly outperformed the same model with Platt scaling~\cite{platt1999probabilistic}, resulting in the promotion of the \ourmethod{}-calibrated model variant to become the primary production model.

Looper also monitors area under the \emph{precision-recall curve (PRAUC)} on online production data. PRAUC improved on $24$ of the $27$ models and was neutral on the rest. Due to data retention, we are able to calculate the exact PRAUC for four of the $27$ Looper models for which we most recently deployed \ourmethod{}, as well as the Brier score~\cite{brier1950verification}. Table~\ref{tab:looper_results} summarizes the results on those models.

\begin{table}[tb]
\centering
\resizebox{\columnwidth}{!}{
    \begin{tabular}{lcccc}
      \toprule
      Model & \multicolumn{2}{c}{PRAUC (\,$\uparrow$\,)} & \multicolumn{2}{c}{Brier Score (\,$\downarrow$\,)} \\
      \cmidrule(lr){2-3} \cmidrule(lr){4-5}
             & \textsc{Platt Scaling} & \ourmethod{} 
             & \textsc{Platt scaling} & \ourmethod{} \\
      \midrule
      \#1 & 0.5295 & \textbf{0.6565} (+23.9\%) & Not available& Not available\\
      \#2 & 0.1824 & \textbf{0.2433} (+33.4\%) & 0.0117 & \textbf{0.0116} (-0.45\%)\\
      \#3 &  0.2161 & \textbf{0.2300} (+6.4\%) & 0.0642 & \textbf{0.0632} (-1.56\%)\\
      \#4 & 0.6268 & \textbf{0.6354} (+1.4\%) & 0.0394 & \textbf{0.0365} (-7.24\%)\\
      \bottomrule
    \end{tabular}
}
\caption{Online results for four models on Looper, where either Platt Scaling or \ourmethod{} is used for post-processing.}
\label{tab:looper_results}
\vspace{-0.1cm}
\end{table}

\paragraph{\textbf{\textsc{MLplatform\#2} Results.}}

\begin{table}[tp]
\centering
\resizebox{\columnwidth}{!}{
\begin{tabular}{r|cccc||cccc}
  \toprule
  \multicolumn{1}{c|}{} & \multicolumn{4}{c||}{\textbf{First Time Period}} & \multicolumn{4}{c}{\textbf{Second Time Period}} \\
  \cmidrule(lr){2-5} \cmidrule(lr){6-9}
  Percentile & Log loss & PRAUC & AUROC & ECE & Log loss & PRAUC & AUROC & ECE \\
  \midrule
  10th & -26.50\% & +3.58\% & +3.15\% & -93.88\% & -16.34\% & +6.31\% & +1.97\% & -90.84\% \\
  25th & -10.35\% & +0.93\% & +0.52\% & -87.45\% & -3.43\% & +1.92\% & +1.10\% & -83.07\% \\
  50th & -1.39\% & +0.32\% & +0.17\% & -53.83\% & -0.86\% & +0.06\% & +0.16\% & -58.19\% \\
  75th & -0.51\% & +0.00\% & +0.03\% & -14.54\% & -0.11\% & -0.95\% & 0.0\% &  -21.96\% \\
  90th & +0.02\% & -2.07\% & -0.20\% & +7.56\% & +1.75\% & -4.85\% & -0.39\% & +26.23\% \\
  \bottomrule
  
\end{tabular}
}
\caption{Summary statistics for the two time periods of the impact of \ourmethod{} on $120$ (first period, left) and $65$ (second period, right) ML models on \textsc{MLplatform\#2}.}
\label{tab:prod_summary_stats}
\vspace{-0.1cm}
\end{table}

This ML training platform is widely used at Meta. ML model training scripts on this platform define the training set, a validation set, the test set, the ML model architecture, and the set of model evaluation metrics. The ML platform allows running an \emph{evaluation flow} that trains the model on the training set and evaluates the model on the test set using the specified set of evaluation metrics, or a \emph{publish flow} that prepares an API endpoint that can generate predictions for live production traffic. We integrated \ourmethod{} into the evaluation flow so that we obtain an additional set of evaluation results for \ourmethod{}.

For every ML model evaluation on this ML platform, we collected parallel evaluation results \emph{with} and \emph{without} \ourmethod{}. We conducted data collection in two separate periods. Table~\ref{tab:prod_summary_stats} summarizes the results. In the first data collection period, we obtained results on $120$ binary classification production models. We found that applying \ourmethod{} post-processing improved log loss for $88.7\%$ of the models, compared to the production models without it. PRAUC improved for $76.7\%$ of the models, AUROC for $80.2\%$, and Expected Calibration Error (ECE)~\cite{naeini2015obtaining} for $86\%$. Note that on most models where no improvement was found, there was no metric degradation either, due to the early stopping. In Table~\ref{tab:prod_summary_stats} we can see that the worst 10th-percentile metric effects of~\ourmethod{} still result in degradations. We manually investigated a sample of those models, and found that these were mostly explained by errors by ML practitioners in the train and test set specification. E.g., in some cases the train and test set were clearly from a different population, or sample weights were defined in the train set but not in the test set.

In a second data collection period, we obtained results on $65$ binary classification models. We found that \ourmethod{} provided improvements over the production models on the log loss metric for $80.3\%$ of the models, on the PRAUC metric for $75.4\%$ of the models, on the AUROC for $79.3\%$, and ECE for $86\%$.



\section{Learnings from Production Deployment}
\label{sec:learnings}

Despite the recent focus of the academic community on multicalibration, it has seen little to no uptake in industry. Here we summarize various learnings from applying multicalibration in industry through production deployments of~\ourmethod{} and of the Multicalibration Error metric (MCE)~\cite{guy2025measuringmulticalibration}.\\

\noindent\textbf{Learning 1. Multicalibration has business value.}\\
Our results demonstrate that \ourmethod{} can improve real-world models' multicalibration as well as performance metrics that tend to correlate with business outcomes, such as PRAUC. Furthermore, baselines such as Isotonic Regression do not achieve this. This contradicts observations by~\citet{hansen2024multicalibration} that (1) Isotonic Regression is often competitive compared to algorithms designed specifically for multicalibration, and (2) practitioners may face a trade-off between multicalibration and predictive performance.\\


\noindent\textbf{Learning 2. Practitioners struggle to define protected groups.}\\
Many multicalibration algorithms~\cite{pmlr-v80-hebert-johnson18a,jin2025} require manual specification of the protected groups. In our experience deploying multicalibration in production, we found that many ML engineers struggle to manually define relevant protected groups, and lack guidelines and frameworks for defining them. This experience is in line with findings from a series of interviews on AI fairness adoption in industry~\cite{holstein2019improving}, where it turns out that practitioners often lack knowledge on different types of biases and cannot identify the relevant biases to correct for. Hence, our experience is that in order to succeed in obtaining industry adoption, a method needs to relieve users of the burden of manually defining protected groups, like~\ourmethod{} does.\\


\noindent\textbf{Learning 3. Practitioners consider adopting multicalibration methods only when they work \emph{out-of-the-box}, reliably, and with no risk of harming model performance.}\\
Many multicalibration algorithms have hyperparameters. We found industry practitioners to be open to applying multicalibration to production ML models, but only if it takes limited effort. Hyperparameter tuning is often seen as complex or time-consuming. They also perceive it as a reliability risk: if a method requires hyperparameter tuning to work well, then they wonder if the hyperparameter configuration will still work well in subsequent training runs.

Therefore, for a multicalibration algorithm to find successful adoption in industry, it is important that it works \emph{out-of-the-box}, without the need for hyperparameter tuning. Moreover, while practitioners care about multicalibration, we found that they are not willing to accept a degradation in metrics like the log loss and the PRAUC, which are often believed to relate to topline metrics. 

We defined a set of hyperparameter default values for~\ourmethod{} that achieves (1) no degradation in PRAUC and log loss, and (2) substantial reduction in MCE. \changed{By pre-computing defaults through meta-analysis across $35$ datasets, we enable \ourmethod{} to deliver consistent improvements without requiring domain expertise or extensive experimentation.} Results in Sections~\ref{sec:benchmark} and~\ref{sec:experiments} confirm that those default values consistently produce log loss and PRAUC improvements over models calibrated with Platt scaling (on Looper) or uncalibrated models (on \textsc{MLplatform\#2}), and consistently yield log loss, PRAUC, and MCE improvements on the public datasets.

\section{Related Work}
\label{sec:related}

\changed{\textbf{Multicalibration} was introduced by \citeauthor{pmlr-v80-hebert-johnson18a}~\cite{pmlr-v80-hebert-johnson18a} as a learning objective requiring predictors to be calibrated across a large collection of subgroups. Subsequent work has deepened its theoretical foundations, exploring computational complexity and connections to other learning paradigms~\cite{detommaso2024multicalibration,NEURIPS2023_7d693203,happymap23,blasiok_et_al:LIPIcs.ITCS.2024.17}.}

\changed{Extensions such as \emph{low-degree multicalibration}~\cite{Gopalan2022LowDegreeM} generalize subgroup definitions using weight functions, showing that restricting to low-degree polynomials can reduce computational complexity. \emph{Swap multicalibration}~\cite{NEURIPS2023_7d693203} strengthens the original concept and is satisfied by some existing algorithms; it is also equivalent to \emph{swap omniprediction} for certain loss function classes.
The notion of \emph{omnipredictors}~\cite{omnipredictors} further generalizes agnostic learning by requiring minimization guarantees for a class of loss functions~\cite{okoroafor25}. Connections to \emph{multi-objective learning} have also been explored, framing multicalibration as a game dynamics problem~\cite{haghtalab2023}.}

\textbf{Algorithms for multicalibration.} The early proposals of multicalibration algorithms, such as HKRR~\cite{pmlr-v80-hebert-johnson18a} (e.g., implemented in MCBoost~\cite{pfisterer2021} or by~\citet{hansen2024multicalibration}) and LSBoost~\cite{pmlr-v202-globus-harris23a}, rely on discretizing the output space and calibrating the prediction for each discretization level until convergence. A discretization-free multicalibration algorithm was proposed in \cite{jin2025}, requiring group functions as input and using a gradient boosting machine to fit an ensemble of depth-two decision trees to the residuals of an uncalibrated predictor using both the original features and the predictions of the uncalibrated predictor. For guaranteeing multicalibration within an additive error, this algorithm requires a loss saturation condition to hold, i.e., that the squared-error loss can be further reduced only by a small amount. Our algorithm uses gradient boosting similarly but does not require pre-specified groups. Instead, it leverages multi-round recursive gradient boosting and supports arbitrary user-defined decision trees in each round.

\changed{\section{Conclusion \& Limitations}}
\label{sec:conclusion}
We tackle the problem of designing a multicalibration algorithm without having access to pre-specified protected sets, that is fast and safe to deploy in production. These three aspects are relevant for practitioners as they are often (1) unable to define  groups for their task, (2) forced to limit the overhead in terms of memory and speed, and (3) required to guarantee that the algorithm does not harm the base predictor. We proposed \ourmethod{}, a novel algorithm that achieves multicalibration by finding and fixing regions of the feature space where the miscalibration is large. Because this might create new regions with evidence of miscalibration, \ourmethod{} employs multiple rounds, where in each round it corrects its own output. In addition, \ourmethod{}'s implementation is fast, lightweight, and safe to deploy. We showed that \ourmethod{} outperforms existing baselines on benchmark datasets and provided results obtained in a large-scale industry production deployment at Meta consisting of hundreds of production models. Finally, we included practical takeaways from our experience applying multicalibration in industry.

\changed{\paragraph{Limitations.}
Although \ourmethod{} shows improvements, several limitations remain. First, guardrails against harming the base predictor may not be sufficient under concept drift. Specifically, \ourmethod{} can potentially worsen performance on post-drift data distributions, as multicalibration is only robust to certain types of distribution shift~\cite{roth22,kim2022universal}. Second, \ourmethod{} is currently limited to binary classification settings. Although multiclass extensions exist, they often suffer from reduced sample efficiency and increased complexity, creating challenges for real-world deployment~\cite{zhao2021calibrating}. Third, while \ourmethod{} can be applied to other data modalities such as text and images (using embeddings), these applications have not been explored in this paper. Future research may study \ourmethod{}’s (1) robustness to broader forms of drift, (2) extension to multiclass/regression problems, and (3) effectiveness on other data modalities.}

\section*{Acknowledgments}
We gratefully acknowledge Nicolas Stier, Udi Weinsberg, Ido Guy, Thomas Leeper, and Mark Tygert for their valuable discussions and support, which contributed to the success of this work.



\bibliographystyle{ACM-Reference-Format}
\bibliography{references}

\appendix
\section{\texorpdfstring{Bridging Multicalibration Error and $\alpha-$MC}{Bridging Multicalibration Error and a-MC}}\label{sec:theory}

Several definitions of approximate multicalibration, known as $\alpha$-multicalibration, exist in the literature. A common approach uses a uniform bound on absolute deviation $\Delta_{h,g}(f) \le \alpha$ where the precise choice of $h$ and $g$ varies \cite{blasiok_et_al:LIPIcs.ITCS.2024.17,happymap23,Gopalan2022LowDegreeM,haghtalab2023}. Another common choice is a uniform bound on the conditional deviation \cite{pmlr-v80-hebert-johnson18a,shabat2020sample}, i.e. \\ $\delta_{h,g}(f):=|\expectation{h(X)g(f(X))(Y-f(X))\mid h(X)g(f(X))=1}| \le \alpha.$
Recent work~\cite{haghtalab2023} improved these by showing a tighter bound $\Delta_{h,g}(f) \le \alpha \sqrt{\expectation{h(X)}}$ can be achieved with no extra requirements.

We next show that Definition~\ref{def:multicalibration}, for the expected deviation corresponding to the sample mean deviation over dataset $D$, i.e.,
$$
\expectationD{h(X)g(f(X))(Y-f(X))}  = \frac{1}{n}\sum_{k=1}^n h(x_k)g(f(x_k)) (y_k - f(x_k)),
$$
with scale parameter chosen 
$\tau_h(f)=\sqrt{\expectationD{h(X)f(X)(1-f(X))}}$,
matches with the MCE metric defined in Eq.~(\ref{equ:mcedef}).

\begin{proposition}\label{prop:mc-mce}
    Given a set $\mathcal{H}$ of group membership functions, a probabilistic predictor $f$ is $\alpha$-multicalibrated  with respect to $\mathcal{H}$, with the scale parameter $\tau_h(f)$, if and only if  $\text{MCE}(f) \le \alpha \sqrt{n}$.
\end{proposition}

\begin{proof} Note that
$
\Delta_{h,g}(f) = \frac{1}{n}\big| \sum_{k=1}^n h(x_k)g(f(x_k))(y_k-f(x_k))\big|.
$
Combining with the definition of $\text{ECCE}_h(f)$, we get that \[\max_{g \in \mathcal{G}}  \Delta_{h,g}(f)  = \frac{n_h}{n} \text{ECCE}_h(f). \]

\noindent Plugging in the definition of $\text{MCE}(f)$, as in Eq.~(\ref{equ:mcedef}), we have
\begin{equation*}
\resizebox{.475\textwidth}{!}{
$\displaystyle
\frac{1}{\sqrt{n}}\text{MCE}(f) \!=\! \frac{1}{\sqrt{n}} \max_{h \in \mathcal{H}} \!\frac{\text{ECCE}_h(f)}{\sigma_h(f)}
\!=\!\! \max_{h \in \mathcal{H}, g \in \mathcal{G}} \!\!\frac{\sqrt{n}  \Delta_{h,g}(f) }{n_h \sigma_h(f)} \!=\! \max_{h \in \mathcal{H}, g \in \mathcal{G}} \!\frac{ \Delta_{h,g}(f)}{\tau_h(f)}.$
}
\end{equation*}
Now, from the last expression, $\alpha$-MC with respect to $\mathcal{H}$ holds if and only if $\text{MCE}(f)\leq \alpha \sqrt{n}$. 
\end{proof}

\section{Convergence Analysis}
\label{app:conv}
We analyze the convergence of Algorithm~\ref{alg:gbmct_new} under specific conditions for the GBM used in each round. We show that (1) the loss does not increase with more rounds and decreases unless the predictor has converged, and (2) the $\alpha$-multicalibration error after $T$ rounds is bounded by the difference between predictors $F_{T+1}$ and $F_T$ and that $0$-multicalibration is achieved at convergence.
Unlike Section~\ref{sec:achieving_multicalibration}, which assumes adding an optimal linear combination of weak learners each round, here the GBM incrementally adds regression trees to $F_t$ to form $F_{t+1}$.



We consider a GBM that, in round $t$, iteratively combines $M_t$ regression trees as 
$F_t^{m+1}(x) = F_t^{m}(x) + \rho_m h_{j_m}(x, f_t(x))
$ where $j_m$ is the index of the regression tree added in the $m$-th iteration, $\rho_m$ is the step size, and $m = 0,1,\ldots,M_t-1$.
Let 
$$
\mathcal{L}_t(w) = \expectationD{\mathcal{L}\left(F_t(X)+\sum_{k=1}^K w_k h_k(X, f_t(X)),Y\right)}.
$$
The gradient of the loss function with respect to $w_k$ is given by
$$
\nabla_k \mathcal{L}_t(w) = \expectationD{h_k(X,f_t(X))(Y-f_{t+1}(X))}
$$
where $f_{t+1}(x) = \mu(F_t(x)+\sum_{k=1}^K w_k h_k(x, f_t(x)))$. 

A GBM can be shown to correspond to a coordinate gradient descent in the coefficient space \cite{randGBM}, with updates of the form
$
w^{m+1} = w^m - \rho \nabla_{j_m} \mathcal{L}_t(w^m) e_{j_m}
$,
where $e_j$ is the $j$-th standard basis vector in $\mathbb{R}^K$ and $\rho$ is the step size. 

\paragraph{Decreasing Loss Property}

\begin{proposition}\label{lem:dec} Assume that the loss function $\mathcal{L}$ is a convex and $L$-smooth (for the log loss, $L=1/4$) and that the step size $\rho$ is set as $\rho = 1/L$. Then, for every $t\geq 0$, the following inequality holds
\begin{equation*}
    \expectationD{\mathcal{L}(F_{t+1}(X), Y)} \leq \expectationD{\mathcal{L}(F_{t}(X), Y)}
\end{equation*}
and is strict if $\ \nabla_{j_m}\mathcal{L}_t(w^m)\neq 0$ for some $m \in \{0,\ldots, M_t-1\}$.
\end{proposition}

\begin{proof}
    Note the following relations, for any $\rho_m$:
\begin{equation*}
\mathcal{L}_t(w^{m+1}) \!\leq \!\mathcal{L}_t(w^m + \rho_m e_{j_m}) \!\leq \!\mathcal{L}_t(w^m) + \rho_m \nabla \mathcal{L}_t(w^m)^\top e_{j_m} \!+\frac{L}{2}\rho_m^2.
\end{equation*}

By taking $\rho_m = - (1/L)\nabla_{j_m} \mathcal{L}_t(w^m)$, we have
$$
\mathcal{L}_t(w^{m+1}) \leq \mathcal{L}_t(w^m) - \frac{1}{2L} (\nabla_{j_m} \mathcal{L}_t(w^m))^2.
$$
It follows that
$
\mathcal{L}_t(w^{M_t})-\mathcal{L}_t(w^0)\leq -\frac{1}{2L}\sum_{m=0}^{M_t-1} (\nabla_{j_m} \mathcal{L}_t(w^m))^2
$,
which allows us to conclude that $\mathcal{L}_t(w^{M_t})-\mathcal{L}_t(w^0)\leq 0$ where the inequality is strict if $\nabla_{j_m} \mathcal{L}_t(w^m)\neq 0$ for some $0
\leq m<M_t$. 
\end{proof}

For a GBM that, in each iteration, greedily selects the weak learner with the largest absolute value of the loss gradient, i.e., 
$
j_m \in \arg\max_j |\nabla_j \mathcal{L}_t(w^m)|, 
$
we can establish a convergence rate to the loss of the optimum linear combination of weak learners using Theorem~4.2~\cite{randGBM}. Let $w^*(t) = \arg\min_w \mathcal{L}_t(w)$. Then, 
\begin{equation}
\expectationD{\mathcal{L}(F_{t+1}(X),Y)}- \mathcal{L}_t(w^*(t))\leq 0.5\cdot \mathrm{Dist_t^2}\frac{1}{M_t}
\label{equ:rate}
\end{equation}
where $\mathrm{Dist}_t$ denotes the distance between the level set $\{w: \mathcal{L}_t(w)\leq \mathcal{L}_t(w^0)\}$ and $w^*(t)$, as defined in~\cite{randGBM}. 

This result justifies modeling each round as updating the predictor with an optimal ensemble of weak learners that minimizes loss. Proposition~\ref{lem:dec} shows the training loss never increases with more rounds and converges to a limit as rounds go to infinity.



\paragraph{A Bound on the Multicalibration Error.} The bound in Proposition~\ref{prop:mc} provides a theoretical upper-bound guarantee, showing that the multicalibration error converges to zero when $C_T$ is bounded by a constant, the gap between prediction values $F_{T+1}$ and $F_T$ vanishes as the number of rounds $T$ increases, and the GBM achieves optimal loss in round $T$. Before proving this, we need a preliminary result:

\begin{lemma}\label{lem:alpha} If $f_T$ is such that $|\expectation{h(X)g(f_T(X))(y-f_T(X))}|\geq \alpha \sqrt{\expectation{h(X)f_T(X)(1-f_T(X))}}$, for some $h\in \mathcal{H}$ and $g\in \mathcal{G}$, then there exists $\tilde{w}\in \mathbb{R}^K$ such that for $\tilde{h}(x) = \sum_{k=1}^K \tilde{w}_k h_k(x,f_T(x))$:
\begin{eqnarray*}
    \alpha \leq \frac{2}{\sqrt{3}}C_T\sqrt{\expectation{\mathcal{L}(F_T(X), Y)}-\expectation{\mathcal{L}(F_T(X)+ \tilde{h}(X),Y)}}\\
    \text{ where }\ C_T^2 \coloneqq \max_{h\in \mathcal{H},g\in \mathcal{G}}\frac{\expectation{g(f_T(X))\mid h(X)=1}}{\expectation{f_T(X)(1-f_T(X))\mid h(X)=1}}.
\end{eqnarray*}
\end{lemma}

\begin{proof}
For any $F$, $h$ and any convex, $L$-smooth loss function $\mathcal{L}$,
\begin{eqnarray*}
\mathcal{L}(F(x),y)-\mathcal{L}(F(x)+h(x),y) \geq \frac{d}{dF} \mathcal{L}(F(x)+h(x),y)(-h(x))\\
\geq \frac{d}{dF}\mathcal{L}(F(x),y)(-h(x)) - L h(x)^2
\end{eqnarray*}
where the first inequality holds as $\mathcal{L}$ is convex and the second inequality holds because $\mathcal{L}$ is $L$-smooth.


For the log loss $(d/dF)\mathcal{L}(F,y) = -(y-\mu(F))$, hence, we have
\begin{equation*}
\resizebox{0.485\textwidth}{!}{
$\expectation{\mathcal{L}(F(X),Y)\!-\!\mathcal{L}(F(X)\!+\!h(X),Y)}
\!\geq\! \expectation{(Y\!-\!f(X))h(X)} \!-\! L \expectation{h(X)^2}$
}
\end{equation*}

Let
$$
(h^*,g^*) \in \arg\max_{h\in \mathcal{G}, g\in \mathcal{G}} \frac{|\expectation{h(X)g(f_T(X))(Y-f_T(X))}|}{\sqrt{\expectation{h(X)f_T(X)(1-f_T(X))}}}
$$
and, let $\tilde{w}$ be defined as:
$$
\tilde{w}_{h,g} = \frac{\mathbf{E}[h(X) g(f_T(X))(Y-f_T(X))]}{\expectation{(h(X) g(f_T(X)))^2}}
$$ 
if $(h,g) = (h^*, g^*)$, otherwise $\tilde{w}_{h,g} = 0$.
For $w = \tilde{w}$, we have both
\begin{eqnarray*}
\expectation{\!(Y\!-\!f_T(X))\!\!\sum_{k=1}^K \!w_k h_k(x,f_T(X))\!}
\!=\! \frac{\expectation{\!h^*(X) g^*\!(f_T(X))(Y\!-\!f_T(X)) }^2}{\expectation{(h^*(X)g^*(f_T(X)))^2}}\\
\expectation{\left(\sum_{k=1}^K w_k h_k(X,f_T(X))\right)^2}
= \frac{\expectation{h^*(X) g^*(f_T(X))(Y-f_T(X)) }^2}{\expectation{(h^*(X)g^*(f_T(X)))^2}}.
\end{eqnarray*}
Hence, we have
\begin{eqnarray*}
&&\expectation{\mathcal{L}(F_T(X),Y) \!-\! \mathcal{L}\left(F_T(X)\!+\!\!\sum_{k=1}^K \tilde{w}_k h_k(X,f_T(X)),\!Y\right)} \\
&\geq & \left(1-L\right)\frac{\expectation{h^*(X) g^*(f_T(X))(Y-f_T(X)) }^2}{\expectation{(h^*(X)g^*(f_T(X)))^2}}\\
&\geq & \left(1-L\right)\frac{\expectation{f_T(X)(1-f_T(X))\mid h^*(X)=1}}{\expectation{g^*(f_T(X))\mid h^*(X)=1}}\alpha^2.
\end{eqnarray*}

Because $L = 1/4$ for the log loss, we conclude the proof
\begin{equation*}
\alpha^2 \leq \frac{4}{3}C_T^2 \left(\expectation{\mathcal{L}(F_T(X),Y)}-\expectation{\mathcal{L}\left(F_T(X)+\tilde{h}(X),Y\right)}\right)
\end{equation*}
\end{proof}

Now, we can derive the upper bound of the MC error:

\begin{proposition}\label{prop:mc} $f_T$ is $\alpha$-MC with
$
\alpha \leq \frac{2}{\sqrt{3}} C_T\sqrt{\phi_T + \epsilon_T}
$
where 
\begin{eqnarray*}
    \phi_T &\coloneqq& 2\ \expectation{|F_{T+1}(X)-F_T(X)|} + \frac{1}{8}\expectation{(F_{T+1}(X)-F_T(X))^2}\\
    \epsilon_T &\coloneqq& \expectation{\mathcal{L}(F_{T+1}(X),Y)}- \mathcal{L}_T(w^*(T)),
\end{eqnarray*}
and $C_T$ is defined as in Lemma~\ref{lem:alpha}.
\end{proposition}

\begin{proof} For any $w\in \mathbb{R}^K$,
\begin{eqnarray*}
\expectation{\mathcal{L}(F_t(X), Y)}-\expectation{\mathcal{L}\left(F_t(X)+ \sum_{k=1}^K w_k h_k(X,f_t(X)), \!Y\right)}\qquad \\
\leq  \underbrace{\!\expectation{ \!\mathcal{L}(F_t(X), \!Y) \!-\!\mathcal{L}(F_{t+1}(X),\!Y)}}_{(*)} \!+\! \underbrace{\expectation{ \!\mathcal{L}(F_{t+1}(X),  \!Y)} \!-\! \mathcal{L}_t(w^*(t))}_{(**)}
\end{eqnarray*}

\noindent The first term ($*$) captures the gap of the loss function values at two successive rounds of the algorithm and can be bounded by $2 \expectation{|F_{t+1}(X)-F_t(X|} + \frac{L}{2}\expectation{(F_{t+1}(X)-F_t(X))^2}$.
The second term ($**$) measures the gap between the loss at the end of round $t$ and the loss from an optimal linear combination of weak classifiers. This gap can be reduced by increasing the number of GBM iterations and is bounded as in Eq~(\ref{equ:rate}). The lemma follows from Lemma~\ref{lem:alpha}.
\end{proof}

%
%
%
The bound on the multicalibration error depends on three terms: (a) $C_T$, which depends inversely on the average prediction variance within subgroups, (b) $\phi_T$, which depends on the gap between the prediction values $F_{T+1}$ and $F_T$, and (c) $\epsilon_T$, which measures the gap between 
the loss achieved by the predictor $F_{T+1}$ and the optimum loss in round $T$.
If $C_T$ is bounded by a constant, the MC error can be made arbitrarily small by reducing $\phi_T$ and $\epsilon_T$. As $T$ increases, $\phi_T$ approaches zero, and $\epsilon_T$ can be minimized by increasing the number of GBM iterations $M_T$, as shown in Eq.~(\ref{equ:rate}). Also, note that
$$
C_T^2 \leq \frac{1}{\min_{h\in \mathcal{H}}\expectation{f_T(X)(1-f_T(X))\mid h(X) = 1}}.
$$
This bound is finite if and only if $\expectation{f_T(X)(1-f_T(X))\mid h(X) = 1} > 0$ for every group $h\in \mathcal{H}$. Equivalently, for every group $h$, there exists at least one member $x$ such that $0 < f_T(x) < 1$. This condition clearly holds if $0 < f_T(x) < 1$ for all $x\in \mathcal{X}$ which is satisfied  whenever $F_T(x) \in [-R, R]$ for every $x\in \mathcal{X}$, for some constant $ R > 0$. Moreover, if the predictor $f_T$ takes values in $[\epsilon,1-\epsilon]$, for some $0 < \epsilon \leq 1/2$, then 
$
C_T^2 \leq \frac{2}{\epsilon}.
$



\section{Extended Experimental Results}
\label{app:extended_results}

\onecolumn

\begin{center}
\renewcommand{\arraystretch}{0.65}
\scriptsize
\begin{tabular}{lllllll}
\toprule
dataset & mcb\_algorithm & $\text{MCE}*\sigma$ (unspecified) & MCE (unspecified) & $\text{MCE}*\sigma$ (prespecified) & MCE (prespecified) & max smECE (prespecified) \\
\midrule
MEPS & base\_model & 0.0265 & 3.9940 & 0.0926 & 3.2849 & 0.0819 \\
 & MCGrad\_msh\_20 & 0.0253 & 3.8815 & 0.0815 & 3.2923 & 0.0755 \\
 & MCGrad & 0.0265 & 3.9940 & 0.0926 & 3.2849 & 0.0819 \\
 & MCGrad\_group\_features & 0.0265 & 3.9940 & 0.0926 & 3.2849 & 0.0819 \\
 & MCGrad\_no\_unshrink & 0.0265 & 3.9940 & 0.0926 & 3.2849 & 0.0819 \\
 & MCGrad\_one\_round & 0.0256 & 3.9250 & 0.0806 & 3.2888 & 0.0733 \\
 & DFMC & 0.0266 & 4.0473 & 0.0870 & 3.4147 & 0.0761 \\
 & HKRR & 0.0318 & 4.8624 & 0.0872 & \textbf{3.1814} & 0.1054 \\
 & Isotonic & \textbf{0.0237} & \textbf{3.6076} & \textbf{0.0797} & 3.4070 & \textbf{0.0713} \\
\midrule
acs\_employment\_all\_states & base\_model & 0.0274 & 87.2951 & 0.0263 & 45.5082 & 0.0499 \\
 & MCGrad\_msh\_20 & \textbf{0.0019} & \textbf{7.1680} & 0.0099 & 4.7747 & 0.0137 \\
 & MCGrad & 0.0085 & 30.7752 & 0.0126 & 14.2049 & 0.0199 \\
 & MCGrad\_group\_features & 0.0143 & 47.9972 & \textbf{0.0061} & \textbf{3.3573} & 0.0120 \\
 & MCGrad\_no\_unshrink & 0.0085 & 31.1755 & 0.0128 & 14.3882 & 0.0201 \\
 & MCGrad\_one\_round & 0.0085 & 30.7752 & 0.0126 & 14.2049 & 0.0199 \\
 & DFMC & 0.0156 & 52.0332 & 0.0110 & 12.6897 & 0.0152 \\
 & HKRR & 0.0196 & 63.8249 & 0.0104 & 14.3973 & \textbf{0.0110} \\
 & Isotonic & 0.0153 & 50.7122 & 0.0210 & 30.8645 & 0.0215 \\
\midrule
HMDA & base\_model & 0.0283 & 13.9421 & 0.0475 & 5.5320 & 0.0468 \\
 & MCGrad\_msh\_20 & 0.0091 & 4.7703 & \textbf{0.0300} & 3.2015 & 0.0365 \\
 & MCGrad & 0.0093 & 4.8958 & 0.0309 & 3.2932 & 0.0371 \\
 & MCGrad\_group\_features & 0.0283 & 13.9249 & 0.0372 & 2.8149 & \textbf{0.0363} \\
 & MCGrad\_no\_unshrink & \textbf{0.0080} & \textbf{4.1592} & 0.0325 & 3.0728 & 0.0405 \\
 & MCGrad\_one\_round & 0.0096 & 4.8650 & 0.0443 & 3.0720 & 0.0504 \\
 & DFMC & 0.0287 & 14.0359 & 0.0372 & \textbf{2.8040} & 0.0380 \\
 & HKRR & 0.0285 & 13.8664 & 0.0401 & 3.3343 & 0.0380 \\
 & Isotonic & 0.0301 & 14.7130 & 0.0468 & 4.6516 & 0.0456 \\
\midrule
acs\_travel\_time\_all\_states & base\_model & 0.0192 & 21.5562 & 0.1736 & 18.6954 & 0.1688 \\
 & MCGrad\_msh\_20 & \textbf{0.0059} & \textbf{6.8154} & 0.0330 & 4.9704 & 0.0346 \\
 & MCGrad & \textbf{0.0059} & 6.8237 & 0.0331 & 4.9730 & 0.0346 \\
 & MCGrad\_group\_features & 0.0165 & 18.6758 & \textbf{0.0297} & \textbf{2.7626} & \textbf{0.0293} \\
 & MCGrad\_no\_unshrink & 0.0061 & 6.9743 & 0.0349 & 4.1746 & 0.0365 \\
 & MCGrad\_one\_round & 0.0118 & 13.3776 & 0.0865 & 12.9763 & 0.0840 \\
 & DFMC & 0.0166 & 18.7304 & 0.0331 & 4.2345 & 0.0323 \\
 & HKRR & 0.0164 & 18.4043 & 0.0355 & 12.7941 & 0.0349 \\
 & Isotonic & 0.0167 & 18.7715 & 0.2116 & 18.1106 & 0.2025 \\
\midrule
acs\_income\_all\_states & base\_model & 0.0233 & 34.0505 & 0.0714 & 17.9622 & 0.0708 \\
 & MCGrad\_msh\_20 & \textbf{0.0049} & 7.9316 & 0.0318 & 6.2694 & 0.0314 \\
 & MCGrad & 0.0050 & 8.1133 & 0.0322 & 6.5280 & 0.0317 \\
 & MCGrad\_group\_features & 0.0201 & 29.6245 & \textbf{0.0106} & 5.3609 & \textbf{0.0152} \\
 & MCGrad\_no\_unshrink & \textbf{0.0049} & \textbf{7.8965} & 0.0317 & \textbf{4.5370} & 0.0305 \\
 & MCGrad\_one\_round & 0.0179 & 27.7918 & 0.0520 & 13.9878 & 0.0502 \\
 & DFMC & 0.0211 & 31.0460 & 0.0185 & 9.9542 & 0.0204 \\
 & HKRR & 0.0230 & 33.8068 & 0.0408 & 10.2374 & 0.0413 \\
 & Isotonic & 0.0216 & 31.7564 & 0.0701 & 17.3586 & 0.0694 \\
\midrule
ACSIncome & base\_model & 0.0265 & 13.4846 & 0.0611 & 6.1578 & 0.0579 \\
 & MCGrad\_msh\_20 & 0.0104 & 6.0717 & 0.0443 & 3.8099 & 0.0498 \\
 & MCGrad & \textbf{0.0102} & \textbf{5.9816} & 0.0439 & 3.8675 & 0.0501 \\
 & MCGrad\_group\_features & 0.0239 & 12.3216 & \textbf{0.0320} & \textbf{3.4824} & \textbf{0.0337} \\
 & MCGrad\_no\_unshrink & 0.0109 & 6.2974 & 0.0446 & 3.7393 & 0.0530 \\
 & MCGrad\_one\_round & 0.0187 & 10.3174 & 0.0489 & 5.6888 & 0.0520 \\
 & DFMC & 0.0243 & 12.5025 & 0.0345 & 3.7487 & 0.0351 \\
 & HKRR & 0.0246 & 12.5983 & 0.0338 & 3.6878 & 0.0354 \\
 & Isotonic & 0.0245 & 12.6166 & 0.0592 & 5.1895 & 0.0574 \\
\midrule
acs\_health\_insurance\_all\_states & base\_model & 0.0198 & 45.2690 & 0.0803 & 38.2892 & 0.0800 \\
 & MCGrad\_msh\_20 & \textbf{0.0044} & \textbf{10.4370} & 0.0082 & 4.0131 & 0.0108 \\
 & MCGrad & \textbf{0.0044} & \textbf{10.4370} & 0.0082 & 4.0131 & 0.0108 \\
 & MCGrad\_group\_features & 0.0120 & 27.8627 & \textbf{0.0038} & \textbf{3.2136} & \textbf{0.0075} \\
 & MCGrad\_no\_unshrink & 0.0045 & 10.5774 & 0.0093 & 3.8013 & 0.0104 \\
 & MCGrad\_one\_round & 0.0096 & 22.4316 & 0.0522 & 21.3152 & 0.0520 \\
 & DFMC & 0.0118 & 27.4016 & 0.0148 & 7.5026 & 0.0232 \\
 & HKRR & 0.0150 & 33.9543 & 0.0322 & 17.3882 & 0.0319 \\
 & Isotonic & 0.0134 & 30.9506 & 0.0780 & 31.8857 & 0.0778 \\
\midrule
acs\_public\_health\_insurance\_all\_states & base\_model & 0.0354 & 40.4553 & 0.1035 & 22.1733 & 0.1025 \\
 & MCGrad\_msh\_20 & 0.0068 & 8.3577 & 0.0324 & 4.9253 & 0.0400 \\
 & MCGrad & \textbf{0.0067} & \textbf{8.2452} & 0.0320 & 4.9934 & 0.0396 \\
 & MCGrad\_group\_features & 0.0282 & 32.7613 & 0.0305 & \textbf{3.4484} & \textbf{0.0324} \\
 & MCGrad\_no\_unshrink & 0.0076 & 9.2211 & \textbf{0.0289} & 5.3953 & 0.0380 \\
 & MCGrad\_one\_round & 0.0265 & 31.3918 & 0.0668 & 12.7760 & 0.0654 \\
 & DFMC & 0.0277 & 32.0694 & 0.0411 & 6.4715 & 0.0434 \\
 & HKRR & 0.0342 & 39.5956 & 0.0428 & 7.3886 & 0.0403 \\
 & Isotonic & 0.0329 & 38.0286 & 0.1064 & 21.8710 & 0.1051 \\
\midrule
acs\_mobility\_all\_states & base\_model & 0.0642 & 52.3930 & 0.0689 & 9.0919 & 0.0652 \\
 & MCGrad\_msh\_20 & 0.0068 & 5.8904 & 0.0306 & 3.1756 & 0.0349 \\
 & MCGrad & 0.0068 & 5.8824 & 0.0302 & 3.2353 & 0.0347 \\
 & MCGrad\_group\_features & 0.0673 & 55.1968 & \textbf{0.0156} & \textbf{2.4286} & \textbf{0.0200} \\
 & MCGrad\_no\_unshrink & \textbf{0.0060} & \textbf{5.2134} & 0.0338 & 2.9911 & 0.0370 \\
 & MCGrad\_one\_round & 0.0150 & 12.7685 & 0.0423 & 6.9712 & 0.0414 \\
 & DFMC & 0.0671 & 54.9004 & 0.0217 & 2.9196 & 0.0261 \\
 & HKRR & 0.0670 & 54.5418 & 0.0406 & 5.2062 & 0.0433 \\
 & Isotonic & 0.0669 & 54.7596 & 0.0659 & 5.2414 & 0.0677 \\
\midrule
CreditDefault & base\_model & 0.0168 & 3.5652 & 0.0783 & 2.1844 & 0.0536 \\
 & MCGrad\_msh\_20 & \textbf{0.0162} & \textbf{3.4353} & 0.0842 & 2.2523 & 0.0543 \\
 & MCGrad & \textbf{0.0162} & 3.4467 & 0.0790 & \textbf{2.1717} & 0.0536 \\
 & MCGrad\_group\_features & 0.0168 & 3.5723 & 0.0772 & 2.3468 & \textbf{0.0526} \\
 & MCGrad\_no\_unshrink & 0.0164 & 3.4813 & 0.0796 & 2.2148 & 0.0536 \\
 & MCGrad\_one\_round & \textbf{0.0162} & 3.4467 & 0.0790 & \textbf{2.1717} & 0.0536 \\
 & DFMC & 0.0164 & 3.4734 & 0.0943 & 2.2036 & 0.0596 \\
 & HKRR & 0.0180 & 3.7981 & \textbf{0.0559} & 2.5684 & 0.0538 \\
 & Isotonic & 0.0168 & 3.5659 & 0.0847 & 2.3044 & 0.0544 \\
\midrule
BankMarketing & base\_model & 0.0195 & 7.0101 & 0.0522 & 5.1177 & 0.0565 \\
 & MCGrad\_msh\_20 & 0.0156 & 5.9140 & 0.0239 & 1.8957 & 0.0328 \\
 & MCGrad & 0.0110 & 4.1535 & 0.0257 & \textbf{1.8693} & 0.0321 \\
 & MCGrad\_group\_features & 0.0163 & 5.8165 & 0.0256 & 2.1236 & 0.0362 \\
 & MCGrad\_no\_unshrink & \textbf{0.0103} & \textbf{3.8441} & \textbf{0.0222} & 2.1485 & \textbf{0.0317} \\
 & MCGrad\_one\_round & 0.0139 & 5.0625 & 0.0384 & 3.0142 & 0.0415 \\
 & DFMC & 0.0158 & 5.6785 & 0.0267 & 1.9211 & 0.0350 \\
 & HKRR & 0.0270 & 9.1271 & 0.0360 & 3.9728 & 0.0390 \\
 & Isotonic & 0.0177 & 6.3438 & 0.0288 & 2.4565 & 0.0370 \\
\bottomrule
\end{tabular}
    \captionof{table}{Multicalibration metrics.}
\end{center}

\begin{center}
\renewcommand{\arraystretch}{0.65}
\scriptsize
\begin{tabular}{lllllll}
\toprule
dataset & mcb\_algorithm & ECCE & $\text{ECCE}/\sigma$ & ECE & Log loss & PRAUC \\
\midrule
MEPS & base\_model & 0.0157 & 2.3774 & 0.0215 & 0.3177 & 0.6211 \\
 & MCGrad\_msh\_20 & 0.0146 & 2.2479 & 0.0186 & \textbf{0.3150} & \textbf{0.6221} \\
 & MCGrad & 0.0157 & 2.3774 & 0.0215 & 0.3177 & 0.6211 \\
 & MCGrad\_group\_features & 0.0157 & 2.3774 & 0.0215 & 0.3177 & 0.6211 \\
 & MCGrad\_no\_unshrink & 0.0157 & 2.3774 & 0.0215 & 0.3177 & 0.6211 \\
 & MCGrad\_one\_round & 0.0140 & 2.1413 & 0.0189 & 0.3153 & 0.6204 \\
 & DFMC & 0.0143 & 2.1821 & 0.0160 & 0.3163 & 0.6116 \\
 & HKRR & 0.0166 & 2.5308 & 0.0252 & 0.3367 & 0.5317 \\
 & Isotonic & \textbf{0.0112} & \textbf{1.7140} & \textbf{0.0118} & 0.3160 & 0.5945 \\
\midrule
acs\_employment\_all\_states & base\_model & 0.0162 & 51.6562 & 0.0328 & 0.2257 & 0.9273 \\
 & MCGrad\_msh\_20 & 0.0010 & 3.8704 & 0.0013 & \textbf{0.1511} & \textbf{0.9597} \\
 & MCGrad & 0.0012 & 4.4955 & 0.0023 & 0.1599 & 0.9512 \\
 & MCGrad\_group\_features & \textbf{0.0008} & \textbf{2.5486} & \textbf{0.0009} & 0.2001 & 0.9391 \\
 & MCGrad\_no\_unshrink & 0.0013 & 4.7459 & 0.0027 & 0.1588 & 0.9519 \\
 & MCGrad\_one\_round & 0.0038 & 13.7094 & 0.0067 & 0.1698 & 0.9451 \\
 & DFMC & 0.0011 & 3.7251 & 0.0025 & 0.2181 & 0.9349 \\
 & HKRR & 0.0031 & 10.1073 & 0.0062 & 0.2111 & 0.9238 \\
 & Isotonic & 0.0009 & 2.9845 & 0.0011 & 0.2066 & 0.9294 \\
\midrule
HMDA & base\_model & 0.0036 & 1.7643 & 0.0063 & 0.2761 & 0.9772 \\
 & MCGrad\_msh\_20 & 0.0044 & 2.3081 & 0.0057 & \textbf{0.2539} & \textbf{0.9818} \\
 & MCGrad & 0.0048 & 2.5294 & 0.0060 & 0.2541 & \textbf{0.9818} \\
 & MCGrad\_group\_features & 0.0029 & 1.4001 & 0.0030 & 0.2746 & 0.9775 \\
 & MCGrad\_no\_unshrink & 0.0034 & 1.7673 & 0.0039 & 0.2546 & 0.9817 \\
 & MCGrad\_one\_round & 0.0033 & 1.6494 & 0.0055 & 0.2609 & 0.9805 \\
 & DFMC & 0.0032 & 1.5643 & 0.0055 & 0.2747 & 0.9775 \\
 & HKRR & 0.0047 & 2.3059 & 0.0069 & 0.2773 & 0.9723 \\
 & Isotonic & \textbf{0.0023} & \textbf{1.1215} & \textbf{0.0024} & 0.2757 & 0.9760 \\
\midrule
acs\_travel\_time\_all\_states & base\_model & 0.0093 & 10.4042 & 0.0181 & 0.6520 & 0.5258 \\
 & MCGrad\_msh\_20 & 0.0028 & 3.2589 & 0.0053 & 0.6304 & 0.5775 \\
 & MCGrad & 0.0028 & 3.2067 & 0.0053 & 0.6303 & 0.5776 \\
 & MCGrad\_group\_features & \textbf{0.0013} & \textbf{1.4963} & 0.0015 & 0.6449 & 0.5337 \\
 & MCGrad\_no\_unshrink & 0.0021 & 2.4608 & 0.0037 & \textbf{0.6294} & \textbf{0.5781} \\
 & MCGrad\_one\_round & 0.0053 & 5.9904 & 0.0068 & 0.6381 & 0.5579 \\
 & DFMC & 0.0025 & 2.8485 & 0.0047 & 0.6460 & 0.5318 \\
 & HKRR & 0.0051 & 5.7510 & 0.0052 & 0.6480 & 0.5083 \\
 & Isotonic & 0.0032 & 3.5899 & \textbf{0.0007} & 0.6487 & 0.5233 \\
\midrule
acs\_income\_all\_states & base\_model & 0.0067 & 9.7652 & 0.0135 & 0.4706 & 0.7407 \\
 & MCGrad\_msh\_20 & 0.0019 & 3.0261 & 0.0029 & \textbf{0.3924} & \textbf{0.8299} \\
 & MCGrad & 0.0020 & 3.1717 & 0.0029 & \textbf{0.3924} & 0.8298 \\
 & MCGrad\_group\_features & 0.0013 & \textbf{1.9035} & \textbf{0.0014} & 0.4608 & 0.7503 \\
 & MCGrad\_no\_unshrink & \textbf{0.0012} & 1.9568 & 0.0018 & 0.3926 & 0.8296 \\
 & MCGrad\_one\_round & 0.0034 & 5.2775 & 0.0049 & 0.4230 & 0.7963 \\
 & DFMC & 0.0019 & 2.7820 & 0.0029 & 0.4648 & 0.7469 \\
 & HKRR & 0.0032 & 4.7310 & 0.0038 & 0.4681 & 0.7230 \\
 & Isotonic & 0.0015 & 2.1644 & 0.0018 & 0.4674 & 0.7402 \\
\midrule
ACSIncome & base\_model & 0.0076 & 3.8599 & 0.0129 & 0.4635 & 0.7898 \\
 & MCGrad\_msh\_20 & 0.0064 & 3.7358 & 0.0120 & 0.3863 & 0.8599 \\
 & MCGrad & 0.0065 & 3.8167 & 0.0114 & 0.3862 & 0.8600 \\
 & MCGrad\_group\_features & \textbf{0.0019} & 1.0030 & 0.0046 & 0.4556 & 0.7963 \\
 & MCGrad\_no\_unshrink & 0.0041 & 2.3710 & 0.0086 & \textbf{0.3853} & \textbf{0.8601} \\
 & MCGrad\_one\_round & 0.0035 & 1.9324 & 0.0080 & 0.4076 & 0.8398 \\
 & DFMC & \textbf{0.0019} & \textbf{0.9585} & \textbf{0.0045} & 0.4568 & 0.7955 \\
 & HKRR & 0.0023 & 1.1808 & 0.0070 & 0.4604 & 0.7721 \\
 & Isotonic & 0.0024 & 1.2285 & 0.0052 & 0.4605 & 0.7872 \\
\midrule
acs\_health\_insurance\_all\_states & base\_model & 0.0091 & 20.7777 & 0.0172 & 0.3971 & 0.2730 \\
 & MCGrad\_msh\_20 & 0.0011 & 2.5741 & 0.0019 & \textbf{0.3762} & 0.3331 \\
 & MCGrad & 0.0011 & 2.5656 & 0.0018 & \textbf{0.3762} & \textbf{0.3334} \\
 & MCGrad\_group\_features & \textbf{0.0005} & \textbf{1.1681} & 0.0008 & 0.3898 & 0.2933 \\
 & MCGrad\_no\_unshrink & 0.0007 & 1.6650 & 0.0015 & \textbf{0.3762} & 0.3330 \\
 & MCGrad\_one\_round & 0.0054 & 12.5842 & 0.0067 & 0.3844 & 0.3072 \\
 & DFMC & 0.0018 & 4.1372 & 0.0055 & 0.3909 & 0.2912 \\
 & HKRR & 0.0073 & 16.6308 & 0.0069 & 0.3944 & 0.2567 \\
 & Isotonic & 0.0009 & 2.0126 & \textbf{0.0007} & 0.3951 & 0.2702 \\
\midrule
acs\_public\_health\_insurance\_all\_states & base\_model & 0.0160 & 18.1469 & 0.0270 & 0.5194 & 0.5979 \\
 & MCGrad\_msh\_20 & \textbf{0.0013} & \textbf{1.6389} & 0.0022 & 0.4682 & 0.6684 \\
 & MCGrad & 0.0014 & 1.7647 & 0.0021 & 0.4681 & 0.6686 \\
 & MCGrad\_group\_features & 0.0018 & 2.0503 & 0.0024 & 0.5082 & 0.6091 \\
 & MCGrad\_no\_unshrink & 0.0020 & 2.3991 & 0.0023 & \textbf{0.4672} & \textbf{0.6703} \\
 & MCGrad\_one\_round & 0.0097 & 11.4243 & 0.0105 & 0.4934 & 0.6328 \\
 & DFMC & 0.0024 & 2.7898 & 0.0029 & 0.5101 & 0.6056 \\
 & HKRR & 0.0054 & 6.2067 & 0.0105 & 0.5147 & 0.5595 \\
 & Isotonic & 0.0019 & 2.2277 & \textbf{0.0012} & 0.5144 & 0.5921 \\
\midrule
acs\_mobility\_all\_states & base\_model & 0.0076 & 6.2231 & 0.0173 & 0.5540 & 0.8059 \\
 & MCGrad\_msh\_20 & 0.0054 & 4.6724 & 0.0080 & \textbf{0.5092} & 0.8694 \\
 & MCGrad & 0.0053 & 4.6375 & 0.0079 & \textbf{0.5092} & \textbf{0.8695} \\
 & MCGrad\_group\_features & 0.0021 & 1.6867 & 0.0036 & 0.5507 & 0.8112 \\
 & MCGrad\_no\_unshrink & 0.0032 & 2.7402 & 0.0050 & \textbf{0.5092} & 0.8691 \\
 & MCGrad\_one\_round & 0.0097 & 8.2706 & 0.0119 & 0.5176 & 0.8612 \\
 & DFMC & 0.0029 & 2.4097 & 0.0035 & 0.5514 & 0.8095 \\
 & HKRR & 0.0048 & 3.9479 & 0.0061 & 0.5532 & 0.7936 \\
 & Isotonic & \textbf{0.0019} & \textbf{1.5359} & \textbf{0.0016} & 0.5521 & 0.8045 \\
\midrule
CreditDefault & base\_model & 0.0077 & 1.6218 & 0.0092 & 0.4333 & 0.5319 \\
 & MCGrad\_msh\_20 & 0.0056 & 1.1842 & 0.0103 & \textbf{0.4325} & \textbf{0.5320} \\
 & MCGrad & 0.0056 & 1.1911 & 0.0110 & 0.4328 & 0.5318 \\
 & MCGrad\_group\_features & 0.0058 & 1.2361 & 0.0118 & 0.4337 & 0.5291 \\
 & MCGrad\_no\_unshrink & \textbf{0.0051} & \textbf{1.0907} & 0.0100 & 0.4327 & 0.5318 \\
 & MCGrad\_one\_round & 0.0056 & 1.1911 & 0.0110 & 0.4328 & 0.5318 \\
 & DFMC & 0.0059 & 1.2568 & 0.0129 & 0.4343 & 0.5275 \\
 & HKRR & 0.0073 & 1.5338 & 0.0134 & 0.4402 & 0.4808 \\
 & Isotonic & 0.0065 & 1.3702 & \textbf{0.0091} & 0.4445 & 0.5251 \\
\midrule
BankMarketing & base\_model & 0.0167 & 5.9821 & 0.0320 & 0.2433 & 0.5587 \\
 & MCGrad\_msh\_20 & 0.0034 & 1.2817 & 0.0062 & 0.2141 & 0.6190 \\
 & MCGrad & 0.0031 & 1.1650 & 0.0051 & \textbf{0.2133} & \textbf{0.6209} \\
 & MCGrad\_group\_features & 0.0035 & 1.2616 & \textbf{0.0039} & 0.2307 & 0.5637 \\
 & MCGrad\_no\_unshrink & 0.0037 & 1.3911 & 0.0053 & 0.2137 & 0.6167 \\
 & MCGrad\_one\_round & 0.0073 & 2.6721 & 0.0126 & 0.2255 & 0.5806 \\
 & DFMC & 0.0035 & 1.2494 & 0.0059 & 0.2303 & 0.5634 \\
 & HKRR & 0.0102 & 3.4546 & 0.0174 & 0.2453 & 0.4933 \\
 & Isotonic & \textbf{0.0030} & \textbf{1.0791} & 0.0060 & 0.2312 & 0.5491 \\
\bottomrule
\end{tabular}
\captionof{table}{(Global) Calibration and model performance metrics.}
\end{center}

\end{document}